\documentclass[twoside]{article}

\usepackage[preprint]{aistats2027}
\usepackage[round]{natbib}

\usepackage{amsmath,amssymb,amsthm,mathtools}
\usepackage{booktabs}
\usepackage{array}
\usepackage{graphicx}
\usepackage{xcolor}
\usepackage[hidelinks]{hyperref}
\usepackage[capitalise,noabbrev,nameinlink]{cleveref}
\crefname{appendix}{Section}{Sections}
\Crefname{appendix}{Section}{Sections}
\crefname{subappendix}{Section}{Sections}
\Crefname{subappendix}{Section}{Sections}

\newtheorem{theorem}{Theorem}
\newtheorem{lemma}[theorem]{Lemma}
\newtheorem{proposition}[theorem]{Proposition}
\newtheorem{corollary}[theorem]{Corollary}
\newtheorem{conjecture}[theorem]{Conjecture}
\theoremstyle{definition}

\newtheorem{remark}[theorem]{Remark}

\crefname{conjecture}{Conjecture}{Conjectures}

\newcommand{\eps}{\varepsilon}
\newcommand{\cC}{\mathcal{C}}
\newcommand{\E}{\mathbb{E}}
\newcommand{\Prob}{\mathbb{P}}
\newcommand{\dH}{d_{\mathrm{H}}}
\newcommand{\ind}{\mathbf{1}}
\newcommand{\Bern}{\mathrm{Bern}}
\newcommand{\MST}{\mathrm{MST}}
\newcommand{\EMST}{\mathrm{EMST}}
\newcommand{\Rraw}{R_{\mathrm{raw}}}
\newcommand{\Rgr}{R_{\mathrm{gr}}}
\newcommand{\RML}{R_{\mathrm{ML}}}

\newcommand{\Wien}{W}             
\newcommand{\Tpi}{T_\pi}          

\hypersetup{
  pdftitle={How Much Does One Wrong Answer Cost? Error Propagation in Constraint Pruning with Noisy LLM Oracles},
  pdfauthor={Zhengtao Gui and Weining Shen}
}

\begin{document}
\runningtitle{Error Propagation in Constraint Pruning}
\runningauthor{Zhengtao Gui and Weining Shen}
\raggedbottom 

\twocolumn[

\aistatstitle{How Much Does One Wrong Answer Cost?\\ Error Propagation in Constraint Pruning with Noisy LLM Oracles}

\aistatsauthor{ \large Zhengtao Gui \hspace{2em} Weining Shen }

\vspace{0.5em}
\aistatsaddress{ \fontsize{11}{14}\selectfont University of California, Irvine } ]

\begin{abstract}
\sloppy
Constraint pruning saves queries to a large language model (LLM) or another noisy oracle by skipping each yes/no question implied by answers accepted earlier, for example through transitivity. Entity-resolution systems and LLM-based rankers use this idea. However, their analyses typically assume a correct oracle. We argue that pruning should be judged by its risk when the oracle errs, since an accepted wrong answer spreads to every question it implies. When each answer is wrong independently with small probability $\varepsilon$, the first-order risk is set by how many facts one wrong answer corrupts on average. We consider two kinds of constraints. For implication rules, as in taxonomies, pruning in a random order adds no first-order risk: a wrong answer that would change $k$ facts is accepted with probability $1/k$, so it costs one fact on average. Rules with $s$ premises break this balance: a wrong answer at a premise shared by $r$ rules can cost order $r^{1-1/s}$ facts. For entity resolution with $m$ records of one entity, the first-order risk equals $\varepsilon$ times the Wiener index (the total pairwise distance) of the spanning tree built by the query order. Pivot orders, which first compare one record with all others, amplify errors by less than 2 and are optimal; random orders by order $m^{1/3}$; the common most-similar-first order is the worst in one dimension and, for large $m$, worse than random in two. With several entities, the cost of wrongly merging two does not depend on how many others exist. Simulations agree with every predicted constant within sampling error.
\end{abstract}

\section{INTRODUCTION}
\label{sec:intro}

Constraint pruning reduces the cost of collecting related yes/no facts from an oracle, such as an LLM or a human annotator, whose answers may be wrong. In entity resolution, a system asks whether two records refer to the same entity and skips pairs implied by transitivity \citep{wang2013leveraging,vesdapunt2014crowdsourcing}. For example, accepting that $A$ and $B$ are the same and that $B$ and $C$ are the same determines the answer for $A$ and $C$ without another query. Similar propagation occurs when enforcing entailment and contradiction constraints between LLM beliefs \citep{mitchell2022concord}, skipping comparisons in LLM ranking \citep{blitzrank2026}, or propagating orientation hints in causal discovery \citep{meek1995causal,long2023causal,mosacd2025}.

However, an accepted error can also propagate to unqueried facts. If $A$, $B$ and $C$ belong to one entity, a correct ``same'' answer for $A,B$ followed by a wrong ``different'' answer for $B,C$ implies a wrong ``different'' answer for $A,C$. In the synthetic citation records of \cref{fig:overview}, one oracle error thus produces two final wrong decisions and saves one query, and the output is still logically consistent. Query counts alone cannot capture this cost. Prior work studies recovery probabilities, assumes deterministic oracles or examines the robustness of query structures \citep{long2023causal,guess2graph2025,blitzrank2026,galhotra2018robust}.

\begin{figure*}[!t]
\centering
\includegraphics[width=\textwidth]{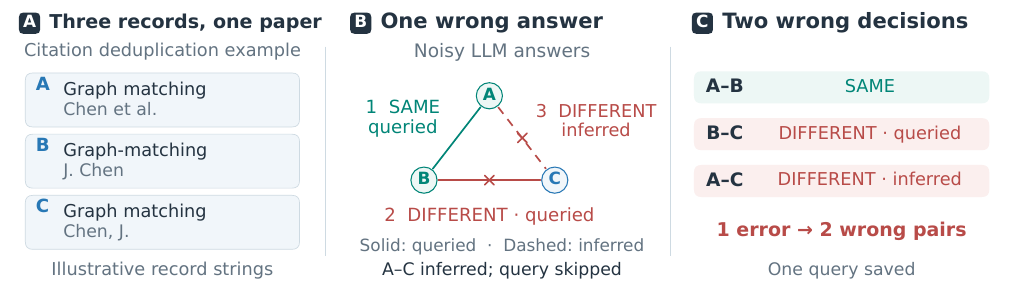}
\caption{Propagation in three synthetic citation records of one publication: one wrong answer produces two wrong pair decisions and saves one query.}
\label{fig:overview}
\end{figure*}

We study \emph{greedy closure decoding}: process questions in an order, skip those implied by accepted answers, and otherwise query and accept the oracle's answer, without revisiting earlier decisions. We measure risk by the expected number of final wrong facts, including inferred ones. The raw baseline asks every question and uses its answer unchanged, without enforcing consistency. Both procedures use the same potential answers: the answers the oracle would give if each question were asked.

In \emph{single-error analysis}, we make one potential answer wrong and keep all others correct. We count the final wrong facts caused by each error location, averaging over query orders when they are random. In \cref{fig:overview}, damage is 2; it is 0 if earlier correct answers imply and skip the wrong question. Under independent errors of small probability $\eps$ \citep{long2023causal,guess2graph2025}, events with two or more errors have probability $O(\eps^2)$ for a fixed finite instance. A reduction lemma expresses risk as $\eps$ times the sum of single-error damages, plus $O(\eps^2)$. The term linear in $\eps$ is the \emph{first-order risk}; its ratio to the raw baseline's $n\eps$ is the \emph{error amplification}, where $n$ counts questions. Random-order ranking and Meek propagation are not analyzed.

Using single-error analysis, we derive the following first-order risk laws:
\begin{itemize}
\item \textbf{Risk conservation for pairwise constraints.} For every satisfiable 2-CNF theory, including implications and pairwise contradictions, uniformly random orders give mean single-error damage 1 for unfixed facts and 0 for facts fixed by the constraints. Pruning therefore does not increase first-order risk over raw decoding (\cref{sec:horn}).
\item \textbf{Error amplification through a shared premise.} For $k$ Horn rules with $s\ge2$ premises each and one shared false premise, we derive the exact mean damage of a single wrong answer there, growing as $\Theta(k^{1-1/s})$ under uniformly random orders, for fixed $s$ (\cref{thm:hub}).
\item \textbf{Relating query order to tree structure.} For $m$ records of one entity, independent errors at a common rate give first-order risk equal to that rate times the noiseless query tree's Wiener index (total pairwise tree distance). Pivot stars minimize this risk among greedy orders, with amplification below 2; random orders give $\Theta(m^{1/3})$, and similarity-first can be worse under \cref{thm:similarity}'s embedding assumptions. We also give bounds for multiple errors at any error rate (\cref{sec:single}).
\item \textbf{Locality of a single false merge.} With only one wrong answer, false-merge damage depends only on the two affected entities. This yields separate within-entity and cross-entity first-order risks. We give bounds, exact small cases and the pivot pipeline's exact first-order formula; the large-cluster growth exponent remains a conjecture (\cref{sec:multi}).
\end{itemize}
We test the theoretical predictions and compare query orders using simulations and LLM answers on the Cora citation data (\cref{sec:numerics}).

\section{RELATED WORK}
\label{sec:related}

\paragraph{Noisy recovery and tree structure.} Work on probabilistic unclean databases relates integrity constraints to minimum-repair and maximum-likelihood decoding \citep{desa2019pud,gilad2023}. Other studies consider noisy graph-label recovery \citep{globerson2015hard,foster2018inference,heidari2019inference,lee2021,jin2026community}, correlation clustering \citep{makarychev2015correlation}, same-cluster queries \citep{mazumdar2017semisupervised} and noisy sorting \citep{feige1994computing,braverman2008noisy}. Their main goals are query efficiency or global recovery guarantees, usually at constant noise \citep{banerjee2022graph}. Our tree-based analysis uses Kruskal's construction \citep{kruskal1956shortest}, extremal properties of the Wiener index \citep{dobrynin2001wiener}, and results on random and spatial spanning trees \citep{addarioberry2009critical,addarioberry2017scaling,aizenman1999scaling}.

\paragraph{Entity resolution and LLM comparisons.} Crowdsourced entity resolution uses transitivity \citep{wang2013leveraging,vesdapunt2014crowdsourcing}. Analyses of query ordering typically assume a correct oracle \citep{firmani2016online,mazumdar2017theoretical}; robust methods add verification queries \citep{galhotra2018robust} or weaken transitivity \citep{gruenheid2015fault}. Recent LLM work studies prompt sensitivity \citep{gui2026evaluating} and matching errors \citep{peeters2025entity}, compares matching strategies \citep{wang2025match}, filters candidates using matching feedback \citep{chen2026generalized}, and balances quality and cost through operator and model selection \citep{guo2026carl}.

Ranking work addresses non-transitive preferences and baseline-sensitive rankings \citep{xu2025investigating}, presentation-order and transitivity inconsistencies \citep{zeng2026rankfusion}, and risks from long inference chains \citep{blitzrank2026}. These studies improve matching, ranking or query allocation. We analyze how an accepted wrong answer causes errors in inferred decisions.

\paragraph{Logical consistency and causal hints.} Belief-correction methods enforce entailment and contradiction relations \citep{mitchell2022concord}. Logic.py translates natural-language puzzles into formal problems for a constraint solver \citep{kesseli2025logic}. Causal discovery methods elicit and propagate hints about edge orientations \citep{long2023causal,mosacd2025,kiciman2023causal}, while Guess2Graph uses unreliable expert guesses to order conditional-independence tests \citep{guess2graph2025}. These methods correct answers or select hints. We study a sequential decoding process that accepts oracle answers in order and keeps accepted answers fixed during subsequent inference.

\section{SETTING AND REDUCTION LEMMA}
\label{sec:setting}

\paragraph{Questions, constraints and noisy answers.} There are $n$ binary questions. Complete vectors satisfying the constraints form the \emph{code} $\cC\subseteq\{0,1\}^n$; each is a feasible \emph{codeword}. They are the models (satisfying assignments) of a propositional theory $\Phi$. The decoder knows the constraints but not the fixed truth $x^\star\in\cC$. We study implication, 2-CNF and definite Horn theories (\cref{sec:horn}), partitions under ``same entity'' transitivity (\cref{sec:er}), and total orders under ``before'' transitivity (\cref{app:proofs-orders}).

The oracle's potential answer $y_i$ to question $i$ is observed only if queried. Thus $y=x^\star\oplus Z$, where $Z_i$ are independent $\Bern(\eps)$ variables: $Z_i=1$ flips the true answer and $Z_i=0$ leaves it unchanged. \Cref{sec:multi} allows two rates. Hamming loss $\dH(\hat x,x^\star)$ counts wrong final answers, including inferred ones; risk $R(\eps)=\E\,\dH(\hat x,x^\star)$ is its expectation. The \emph{raw} decoder asks all questions and returns $y$, with risk $\Rraw=\eps n$; the output may violate constraints.

The \emph{greedy} decoder $G_\pi$ follows order $\pi$, storing accepted question--answer pairs in an initially empty set $L$. Among codewords agreeing with $L$, if all give the same answer at $i$, use it and skip the query. Otherwise both answers are feasible: query the oracle and add $(i,y_i)$ to $L$, keeping this answer fixed. A wrong answer can thus be accepted while preserving feasibility. After all questions, one codeword remains (\cref{lem:welldef}), also obtained by the accept-if-consistent rule. Unless stated otherwise, $\pi$ is uniform and independent of $Z$. Implication tests take polynomial time without enumerating $\cC$: satisfiability tests for 2-CNF and Horn theories, union-find for partitions and transitive closure for orders. The \emph{maximum-likelihood} (ML) decoder uses all answers jointly, returning a codeword closest to $y$ in Hamming distance, with uniform tie-breaking.

\paragraph{Single-error damage.} For a decoder $\hat x$ whose output depends on $y$ and on an independent seed $\pi$, the \emph{single-error damage} at question $a$ is the expected loss when only the answer to $a$ is wrong,
\[
D(a)=\E_\pi\,\dH\big(\hat x(x^\star\oplus e_a,\pi),\,x^\star\big),
\]
where $e_a$ is the $a$-th unit vector. For the greedy decoder we write $D_\pi(a)=\dH(G_\pi(x^\star\oplus e_a),x^\star)$ for a fixed order, so that $D(a)=\E_\pi D_\pi(a)$. The \emph{first-order constant} is $\kappa=\lim_{\eps\to0}R(\eps)/\eps$, and the \emph{amplification} of a decoder is $\kappa/\kappa_{\rm raw}=\kappa/n$, the mean damage of one wrong answer relative to the raw decoder.

For $x\in\cC$ let $m_a(x)$ be the number of codewords of the form $x\oplus e_a\oplus e_b$, $b\neq a$, when $x\oplus e_a\notin\cC$ (and $0$ otherwise). In that case they are the codewords other than $x$ that tie with $x$ as nearest to $x\oplus e_a$.

\begin{table*}[t]
\centering
\small
\caption{Mean single-error damage; skipped errors count as zero. Logical rows use uniform orders; partition rows average over all within- or cross-cluster pairs, giving first-order amplification over raw risk. Optimality is among greedy orders at first order. Conj.: conjectured growth.}
\label{tab:summary}
\setlength{\tabcolsep}{3pt}
\begin{tabular}{@{}>{\raggedright\arraybackslash}p{.285\textwidth}>{\raggedright\arraybackslash}p{\dimexpr.515\textwidth-4\tabcolsep\relax}>{\raggedright\arraybackslash}p{.20\textwidth}@{}}
\toprule
Setting & Mean damage (wrong final facts) & Reference\\
\midrule
2-CNF, implication rules & $1$ for unfixed facts, $0$ for fixed facts (backbone); ML: $\le1$ & Cor.~\ref{cor:conservation}; Prop.~\ref{prop:ml2cnf}\\
Horn hub: $k$ rules, $s\ge2$ premises & $\sim\Gamma(1+\frac1s)k^{1-1/s}$ at the shared premise; ML: $0$ ($k\ge2$) & Thm.~\ref{thm:hub}; Rem.~\ref{rem:singleflipfails}\\
One entity, $m$ records & pivot $2(m-1)/m$ (optimal); random $\Theta(m^{1/3})$; similarity-first under Thm.~\ref{thm:similarity}: $\gtrsim(m/\log m)^{1/d}$ & Props.~\ref{prop:star}, \ref{prop:random}; Thm.~\ref{thm:similarity}\\
Cross pair, $s$ records per entity & random $g(s)\in[\frac{s-1}{2s-1},13s]$, conj.\ $\Theta(s^{2/3})$; pivot $(2s-1)/s$ & Props.~\ref{prop:gbounds}, \ref{prop:pivotmulti}; Conj.~\ref{conj:contamination}\\
Total orders & randomized quicksort $\to2$ & Prop.~\ref{prop:quicksort}\\
\bottomrule
\end{tabular}
\end{table*}

\begin{lemma}[Reduction lemma]\label{lem:reduction}
Let $\hat x$ be any decoder that returns $x^\star$ when there is no wrong answer. For every $\eps\in[0,\tfrac12]$,
\[
\Big|R_{\hat x}(\eps)-\eps\sum_{a=1}^nD(a)\Big|\le\tfrac32\,\eps^2n^2(n-1).
\]
Moreover:
(i) ML has $D_{\rm ML}(a)=1$ if $x^\star\oplus e_a\in\cC$ and $D_{\rm ML}(a)=2m_a/(1+m_a)$ otherwise.
(ii) The greedy decoder has first-order constant $\sum_a\E_\pi D_\pi(a)$, so its amplification is the mean damage of one wrong answer.
\end{lemma}

The remainder grows with $n$, so the expansion describes a fixed instance as $\eps\to0$.

\section{IMPLICATION THEORIES AND HORN RULES}
\label{sec:horn}

We apply the single-error analysis of \cref{sec:setting} to logical rules to identify when constraint propagation amplifies errors. We first show that, under uniformly random query orders, implication rules and, more generally, 2-CNF constraints do not increase first-order risk over raw decoding; a shared-premise Horn family then shows why this conclusion does not hold in general.

Let $\Phi$ be a satisfiable theory over atoms $V$, $|V|=n$, each representing a binary fact. Write the truth $x^\star$ as a satisfying assignment $\sigma^\star$ here. A literal is an atom or its negation. Let $t_x=x$ when $\sigma^\star(x)=1$ and $t_x=\neg x$ otherwise, with $\bar t_x$ the opposite literal. The \emph{backbone} $B=\{x:\Phi\models t_x\}$ contains facts fixed by the rules ($\models$ denotes logical entailment). For $a\notin B$, the \emph{flip set}
\[
K_a=\{x\in V:\ \Phi\wedge\bar t_a\models\bar t_x\}
\]
contains $a$ and every atom whose value is forced to change by a wrong value at $a$. We say that $a$ has the \emph{single-flip property} if flipping exactly the atoms of $K_a$ in $\sigma^\star$, while leaving all others unchanged, gives another model of $\Phi$.

\begin{theorem}[Damage of one wrong answer]\label{thm:singleflip}
Suppose that the only wrong answer is at atom $a$. If $a\in B$, the greedy decoder returns the truth for every order. If $a\notin B$ has the single-flip property, then for every order $\pi$
\[
D_\pi(a)=|K_a|\cdot\ind\{a\text{ comes first in }\pi\text{ among }K_a\},
\]
so $\E_\pi D_\pi(a)=1$ for a uniformly random (or any exchangeable) order.
\end{theorem}

In the second case, a correct value on another atom of $K_a$ implies $a$'s true value and skips its query. Only if $a$ comes first is the error accepted, causing $|K_a|$ wrong facts. Its probability $1/|K_a|$ yields mean damage 1; individual runs have damage 0 or $|K_a|$.

\begin{proposition}\label{prop:twocnf}
If every clause of $\Phi$ has at most two literals, every atom outside the backbone has the single-flip property.
\end{proposition}

\Cref{prop:twocnf} holds because, if a flip makes one literal of a two-literal clause false, $\Phi\wedge\bar t_a$ forces the other to be true, preserving the clause.

\begin{corollary}[Conservation law]\label{cor:conservation}
Let $\Phi$ be a satisfiable 2-CNF theory with backbone $B$. For every $\eps\in[0,\tfrac12]$,
\[
-\eps^2n(n-1)\ \le\ \Rgr(\eps)-\eps\,(n-|B|)\ \le\ \tfrac12\eps^2n^2(n-1).
\]
In particular, for every implication graph $G=(V,A)$, with $\Phi_G=\bigwedge_{(u,v)\in A}(\neg u\vee v)$, greedy decoding has $\Rgr(\eps)=\eps n+O(\eps^2)=\Rraw(\eps)+O(\eps^2)$, for false-positive and false-negative errors alike and for every truth.
\end{corollary}

Thus the first-order risk is $\eps(n-|B|)$; implication graphs have no fixed facts and match raw at first order. Multiple errors can still cause differences. If accepted, a false positive (false to true) changes false descendants, and a false negative (true to false) changes true ancestors. The order matters: on $v_1\to\dots\to v_n$ queried as $v_1,\dots,v_n$, a single false positive is accepted and changes every later false fact.

\begin{proposition}[ML on 2-CNF]\label{prop:ml2cnf}
On a satisfiable 2-CNF theory, the ML single-error damage at $a\notin B$ is 1 if $|K_a|\le2$ and 0 if $|K_a|\ge3$. Hence $\RML(\eps)=\Rgr(\eps)-\eps\,\#\{a\notin B:|K_a|\ge3\}+O(\eps^2)$.
\end{proposition}

\paragraph{Horn rules.} The \emph{Horn hub} $H_{k,s}$ has $k$ rules $b_{i,1}\wedge\dots\wedge b_{i,s-1}\wedge a\to c_i$ ($i=1,\dots,k$), sharing premise $a$. Each rule's other premises and conclusion occur in no other rule. All $b_{i,j}$ are true and $a,c_i$ are false, so $n=1+ks$. For $s=2$ this is the pattern $\mathrm{isa}(x,y)\wedge\mathrm{has}(y,z)\to\mathrm{has}(x,z)$ with a shared premise.

\begin{theorem}[Horn hub]\label{thm:hub}
For a single false positive at the shared premise $a$ of $H_{k,s}$, $s\ge2$, and a uniformly random order,
\begin{align*}
\E_\pi D_\pi(a)&=(k+1)\int_0^1(1-u^s)^k\,du\\
&=\frac{(k+1)\,\Gamma(1+\frac1s)\,\Gamma(k+1)}{\Gamma(k+1+\frac1s)},
\end{align*}
which equals $\Gamma(1+\frac1s)\,k^{1-1/s}(1+O(1/k))$. Every other single error on $H_{k,s}$ has damage exactly 1. Hence $\Rgr(\eps)=\Rraw(\eps)+\eps\,[\E_\pi D_\pi(a)-1]+O(\eps^2)$.
\end{theorem}

Here $\Gamma$ is the Gamma function. The mechanism differs from \cref{thm:singleflip}: a wrong $a$ alone implies no other values. It is skipped only if one rule's other $s$ atoms were processed earlier: their true premises and false conclusion force $a$ false. If accepted, each rule's last atom is inferred wrongly and skipped; it can be a premise or the conclusion. Including $a$, the total damage is $k+1$.

Appendix~\ref{app:proofs-horn} gives the ML proof and the arrival-time derivation of the Horn formula, including its $s=2$ special case.

For Horn theories with at most $s$ premises per rule, we conjecture that mean false-positive damage at $a$ is bounded by $C_s\,c^{1-1/s}$, where $c$ counts newly true atoms after adding $a$ to the truth, including $a$ (\cref{conj:bodysize}). The hub shows this exponent cannot be lowered.

Under uniformly random query orders, acceptance probability and propagation damage balance for 2-CNF constraints, so pruning does not increase first-order risk over raw decoding. The shared-premise Horn family shows that this balance can fail, with the mean cost of one wrong answer growing with the number of rules.

\section{ENTITY RESOLUTION WITH TRANSITIVITY PRUNING}
\label{sec:er}

We now apply the greedy-decoding framework to entity resolution, where each pairwise answer states whether two records refer to the same entity. Transitivity lets ``same'' on $A$--$B$ and $B$--$C$ imply ``same'' on $A$--$C$, saving a query but potentially spreading an error. This involves three pairwise facts, so the 2-CNF conservation law in \cref{sec:horn} does not directly apply. We first study false splits within one entity, then false merges between entities.

\subsection{One Cluster}
\label{sec:single}

The $m$ records in $[m]$ all refer to one entity, so every pair's true answer is ``same''. The algorithm does not know this: its feasible outputs are all partitions of $[m]$, each specifying which records belong together. There are $\binom m2$ pairwise questions. We write the error rate as $\delta$: each answer is independently flipped to ``different'' with probability $\delta$. The loss $L_\pi$ counts pairs incorrectly separated in the final output.

The \emph{union-find pipeline} $U_\pi$ \citep{wang2013leveraging} maintains current groups of records, called blocks, and accepted ``$\neq$'' constraints between them. A block need not coincide with a true entity. It scans pairs in order $\pi$. For $\{u,v\}$, membership in one block already implies ``same'', while a $\neq$ constraint between their blocks implies ``different''; either case skips the query. Otherwise it asks the oracle: ``same'' merges the blocks, and ``different'' records a constraint. Existing constraints are carried over when blocks merge. \Cref{lem:ufgreedy} shows that this is the greedy decoder on the partition code.

Construct the \emph{Kruskal tree} $\Tpi$ from the order alone, without using the noisy answers. Scan pairs in order $\pi$, retaining a pair only when it connects two components of the edges retained so far. With all answers correct, these $m-1$ tree edges are exactly the queried pairs. For a tree $T$, let $d_T(u,v)$ count edges on the unique path from $u$ to $v$. Its Wiener index \citep{wiener1947structural} is
\[
\Wien(T)=\sum_{\{u,v\}}d_T(u,v)=\sum_{e\in T}n_e(m-n_e)
\]
where deleting edge $e$ splits the tree into parts of sizes $n_e$ and $m-n_e$. Exactly $n_e(m-n_e)$ pairwise paths use that edge, which gives the second equality.

\begin{lemma}[Sandwich]\label{lem:sandwich}
Fix any order $\pi$ and any pattern of wrong answers. For a pair $\{u,v\}$, let $k(u,v)$ be the number of edges of $\Tpi$ whose answer is wrong on the tree path from $u$ to $v$. If $k(u,v)=0$, then $u$ and $v$ end in the same block; if $k(u,v)=1$, they end in different blocks. Hence $N_1\le L_\pi\le N_{\ge1}$, where $N_j$ ($N_{\ge j}$) is the number of pairs with $k=j$ ($k\ge j$).
\end{lemma}

Every tree edge is queried even when errors occur, and each pipeline block stays inside a component of the forest built so far by the Kruskal scan. This yields \cref{lem:sandwich}. Paths with two or more wrong edges need not lead to separation; with a single wrong answer, the two bounds coincide and give the exact loss.

\begin{theorem}[Wiener law]\label{thm:wiener}
For every order $\pi$ and every $\delta\in[0,1]$,
\[
\delta\Wien(\Tpi)-\delta^2\Wien_2(\Tpi)\ \le\ \E L_\pi\ \le\ \delta\Wien(\Tpi),
\]
where $\Wien_2(T)=\sum_{\{u,v\}}d_T(u,v)^2$. A single wrong answer on a tree edge $e$ costs exactly $n_e(m-n_e)$ pairs, and a single wrong answer on any other pair costs nothing.
\end{theorem}

\Cref{thm:wiener} makes the effect of the query tree explicit: for fixed $m$, $\E L_\pi=\delta\Wien(\Tpi)+O(\delta^2)$ as $\delta\to0$, with the displayed bounds valid for every $\delta\in[0,1]$. Raw decoding has risk $\delta\binom m2$, so first-order amplification is $\Wien(\Tpi)/\binom m2$, the average tree distance. The single-error cut identity and \cref{lem:sandwich} hold for any error placement; the expectation uses independent errors with a common rate. \Citet{galhotra2018robust} previously identified the vulnerability of tree-based inference to errors; here the bounds quantify it for this decoder.

\begin{proposition}[Pivot orders are optimal]\label{prop:star}
(i) For every spanning tree $T$ of $K_m$, $m\ge3$, $\Wien(T)\ge(m-1)^2$, with equality if and only if $T$ is a star (a classical property of the Wiener index; \citealp{dobrynin2001wiener}). So the first-order risk $\delta\Wien(\Tpi)$ is minimized exactly by the orders whose Kruskal tree is a star, such as asking every pair with a fixed pivot first \citep{vesdapunt2014crowdsourcing}; their amplification is $2(m-1)/m<2$.
(ii) Among decoders that observe the answers on a fixed spanning tree $T$ only (non-adaptive designs with $m-1$ queries, each observed answer flipped independently with probability $\delta$), if $\delta\le1/m$ the worst-case risk over the one-cluster truth and the $m-1$ two-block partitions obtained by cutting one edge of $T$ is at least $\frac12\delta(1-\delta)^{m-2}(m-1)^2$. So, to first order in $\delta$, the worst-case risk of every such design is at least half of the pivot order's risk $\delta(m-1)^2$ at the one-cluster truth.
\end{proposition}

\begin{proposition}[Random order]\label{prop:random}
For a uniformly random order, $\Tpi$ is the minimum spanning tree $\MST_m$ of $K_m$ with i.i.d.\ weights, and $\E L=\delta\,\E\,\Wien(\MST_m)+O(\delta^2)$ with $\E\,\Wien(\MST_m)=\Theta(m^{7/3})$, an amplification $\Theta(m^{1/3})$. At $\delta=x\,m^{-1/3}$ with $x>0$ fixed, the expected fraction of lost pairs stays bounded away from 0 and 1 as $m\to\infty$.
\end{proposition}

Independent uniform edge weights generate a uniform pair order; known tree-distance estimates \citep{addarioberry2009critical,addarioberry2017scaling} then yield \cref{prop:random} through \cref{lem:sandwich}.

\begin{theorem}[Most-similar-first order]\label{thm:similarity}
Let items carry i.i.d.\ embeddings on $[0,1]^d$ with a density bounded below by $f_0>0$, let $\mu_f$ be the mean distance between two items, and let pairs be asked in increasing order of distance. Then $\Tpi$ is the Euclidean minimum spanning tree, and for every $m\ge2$ with $f_0m\ge2\log m$ (natural logarithm),
\[
\frac{\E\,\Wien(\EMST_m)}{\binom m2}\ \ge\ \frac{\mu_f-\sqrt d/m}{4\sqrt d}\Big(\frac{f_0\,m}{2\log m}\Big)^{1/d}.
\]
For $d=1$ the amplification is exactly $(m+1)/3$.
\end{theorem}

Under these embedding assumptions, similarity-first has larger asymptotic amplification than random order in one and two dimensions. Nearby connections can create long tree paths, giving a single error more opportunities to separate a pair. In one dimension, the tree is a path, which maximizes the Wiener index \citep{dobrynin2001wiener}. Appendix~\ref{app:proofs-er} derives the tree estimates and quantitative growth-rate comparisons. With unequal error rates, semantic similarity-based querying \citep{wang2013leveraging} also depends on which pairs are answered wrongly (\cref{sec:numerics}).

\subsection{Several Clusters}
\label{sec:multi}

Now records belong to $K$ true clusters of sizes $s_1,\dots,s_K$, with $n=\sum_i s_i$ records (here $n$ counts records, not pairwise facts). All $\binom n2$ pairs are scanned in a uniformly random order. Answers are flipped independently: a within-cluster pair is wrongly answered ``different'' with probability $\eps_w$, and a cross-cluster pair ``same'' with probability $\eps_x$. Write $R_U$ for the expected number of wrong final pair decisions, counting both false merges and false splits.

To isolate a false merge, consider only two clusters of sizes $s_A,s_B$. Fix a cross pair and make its potential answer ``same'', with all other answers correct. Let $g(s_A,s_B)$ be its mean damage over random orders, counting a skipped wrong answer as zero, and write $g(s)=g(s,s)$. Accepting the error can also block later correct within-cluster merges. Thus $g$ includes both wrong cross-cluster merges and possible splits within these two clusters.

\begin{theorem}[Exact first-order decomposition]\label{thm:multicluster}
With $\eps=\max(\eps_w,\eps_x)$,
\begin{align*}
R_U={}&\eps_w\sum_{i=1}^K\E\,\Wien(\MST_{s_i})\\
&+\eps_x\sum_{i<j}s_is_j\,g(s_i,s_j)+O(\eps^2).
\end{align*}
In particular the damage of a false merge does not depend on $K$ or on the other clusters.
\end{theorem}

The first term sums false-split risks within the true clusters, applying the Wiener law to each. The second sums false-merge risks: clusters $i,j$ have $s_is_j$ cross pairs, each contributing mean damage $g(s_i,s_j)$. Correct cross-cluster answers confine a single error to its incident clusters. Simultaneous errors can involve more clusters, so this locality underlies the first-order decomposition.

The following bounds on $g(s)$ count skipped errors as zero.

\begin{proposition}[Size of one false merge]\label{prop:gbounds}
Let $q$ be the event that the false pair is asked. For $s\ge2$, $\frac1{2s-1}\le\Prob(q)\le\frac{12.4}{s+1}+\frac12\big(\frac9{16}\big)^{s-1}$ and $\frac{s-1}{2s-1}\le g(s)\le13\,s$. Under standard local-limit facts for Erdős–Rényi graphs, $s\,\Prob(q)\to C_q\approx1.4286$ (\cref{prop:cq}).
\end{proposition}

Exact computation gives $g(2)=29/18$ and $g(3)=313913/150150\approx2.091$ (\cref{tab:gexact}). For comparison, the pivot pipeline processes records one at a time. It compares each new record with the first record of each existing block, in block-creation order, until ``same'' lets it join that block; if all answers are ``different'', it creates a new block \citep{vesdapunt2014crowdsourcing}. Its first-order risk is given by \cref{prop:pivotmulti}.

\begin{proposition}[Pivot pipeline]\label{prop:pivotmulti}
For $K$ clusters of size $s$, with items arriving in a uniformly random order,
\[
R_{\rm piv}=\eps_wK(s-1)^2+\eps_x\binom K2s(2s-1)+O(\eps^2).
\]
\end{proposition}

The coefficient of $\eps_x$ in \cref{prop:pivotmulti}, divided by the $\binom K2s^2$ cross pairs, gives mean damage $(2s-1)/s$ per potential wrong cross-pair answer. The random-order mean is $g(s)$, giving a loss ratio $g(s)\,s/(2s-1)$.

\begin{conjecture}[Contamination law]\label{conj:contamination}
$g(s)=\Theta(s^{2/3})$. Consequently, for $K\ge2$ clusters of size $s$, $R_U=\Theta(\eps_wKs^{7/3})+\Theta(\eps_xK^2s^{8/3})$ at first order.
\end{conjecture}

Simulation supports $s^{2/3}$ growth but does not determine the asymptotic constant (\cref{fig:main}, lower panel C; \cref{tab:gmc}). A proof would need a scaling limit describing the constrained merging process in the critical window of the within-cluster random graphs.

In entity resolution, the noiseless query tree determines one-cluster first-order risk; with independent errors at a common rate, pivot stars minimize this risk among greedy query orders. With several entities, a single error remains confined to the one or two entities involved, yielding separate first-order contributions from within-entity and cross-entity errors.

\section{EXPERIMENTAL VALIDATION}
\label{sec:numerics}

We check the formulas in simulation, compare query orders at higher error rates allowing multiple wrong answers, and evaluate final losses using LLM answers on Cora. Loss counts wrong final decisions, including inferred ones. In record matching, false splits separate records of the same entity, while false merges group records of different entities. Amplification is the final wrong-pair count divided by the raw baseline's wrong-answer count. All single-error averages count skipped errors as zero.

\Cref{fig:main} connects the propagation examples with numerical checks of the preceding laws. (A) contrasts 2-CNF conservation with shared-premise Horn amplification. In (B), splitting eight records into $1+7$ or $4+4$ produces 7 or 16 wrong pairs. Its lower panel compares first-order amplification, equal to mean noiseless-tree distance under independent, equal-rate errors. In (C), one false merge can cause both false merges and false splits within the two involved entities; the lower panel compares $g(s)$ with pivot's $(2s-1)/s$. The $2.38m^{1/3}$ and $0.9s^{2/3}$ curves are guides; the finite-size slope $\approx0.43$ in (B) does not determine the asymptotic exponent, and (C) does not establish its asymptotic constant.

\begin{figure*}[!t]
\centering
\includegraphics[width=.94\textwidth]{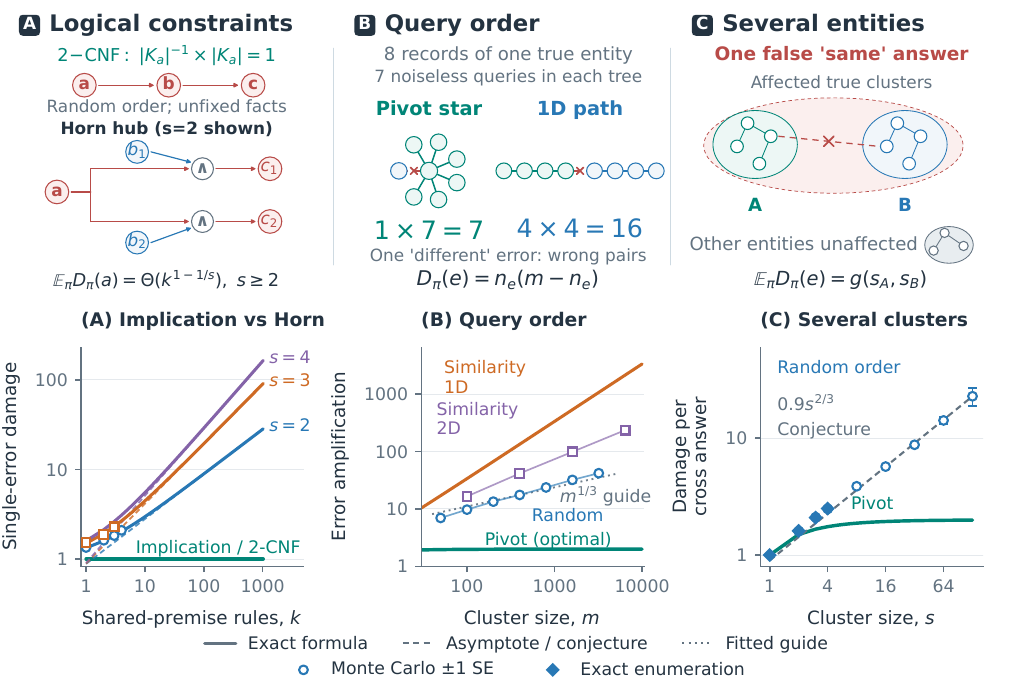}
\caption{Single-answer error propagation (top) and numerical checks (bottom): (A) logical constraints, (B) within-entity query order, (C) cross-entity false merges. Horn example: $k=s=2$, $a,c_i$ false, $b_i$ true. Shading marks affected true clusters, not the final partition. Thick solid: exact formulas; thin connectors: visual guides; dashed: asymptotes (A), conjecture (C); dotted: fitted guide (B). Open markers: simulations $\pm1$ standard error; diamonds: exact enumeration.}
\label{fig:main}
\end{figure*}

\paragraph{Checking the analytical predictions.} Enumeration and direct decoding test conservation, Horn damage and false-merge locality independently of the formulas. The conservation estimates are within 2.66 Monte Carlo standard errors (MCSE) of their predictions; the other tabulated estimates are within 1.62 MCSE. At $\eps=10^{-3}$, dividing the full decoder's mean false-merge (FP) and false-split (FN) counts by $\eps$ also gives estimates close to the first-order coefficients. Appendix~\ref{app:numerics} gives the full comparison (\cref{tab:checks}), simulation settings and statistical details.

\paragraph{Higher error rates.} We also test the full decoder where multiple errors may interact. One entity has $m=100$ records, so every pair's correct answer is ``same''. Each answer is independently flipped to ``different'' at rate $\delta=1\%$--$20\%$, with 200 runs per configuration. \Cref{tab:highnoise} reports the fraction of pairs wrongly separated: at a 10\% answer-error rate, this fraction averages 17.4\% for pivot and 45.8\% for random order. Across the tested rates, random-order mean loss is $1.8$--$4.3$ times the pivot value; at $m=200$ the ratios are $1.9$--$5.4$ (\cref{tab:highnoise-scale}). Pivot's advantage therefore persists in these settings. At higher rates, the first-order sum overestimates random-order loss: several errors may affect the same pairs, which are counted only once.

\begin{table}[t]
\centering
\small
\caption{Fraction of wrongly separated pairs on one cluster of $m=100$ records when each answer is flipped to ``different'' with probability $\delta$. Means $\pm1$ Monte Carlo standard error over 200 runs; Random/Pivot is the ratio of the mean losses.}
\label{tab:highnoise}
\setlength{\tabcolsep}{4pt}
\begin{tabular}{@{}llll@{}}
\toprule
$\delta$ & Pivot & Random & Random/Pivot\\
\midrule
0.01 & $0.019\pm0.002$ & $0.082\pm0.008$ & 4.2\\
0.02 & $0.039\pm0.002$ & $0.170\pm0.012$ & 4.3\\
0.05 & $0.094\pm0.003$ & $0.320\pm0.010$ & 3.4\\
0.10 & $0.174\pm0.004$ & $0.458\pm0.008$ & 2.6\\
0.20 & $0.335\pm0.004$ & $0.604\pm0.006$ & 1.8\\
\bottomrule
\end{tabular}
\end{table}

\begin{table}[!t]
\centering
\small
\caption{Amplification on Qwen2.5-1.5B answers for Cora: final wrong pairs per wrong answer. Wrong: answer-error rate. Pivot/random use 200 orders in rows~1 and~4; rows~2--3 use 30 matched orders before and after uniformly redistributing errors within each entity, preserving its error count.}
\label{tab:realllm}
\setlength{\tabcolsep}{3pt}
\begin{tabular}{@{}lcccc@{}}
\toprule
Setting & Wrong & Pivot & Random & Sim.-first\\
\midrule
One entity & 6.6\% & 1.23 & 2.31 & 0.85\\
One entity, 30 orders & 6.6\% & 1.37 & 2.19 & 0.85\\
Shuffled, 30 orders & 6.6\% & 1.73 & 3.55 & 3.48\\
Multi-entity blocks & 6.3\% & 1.14 & 1.46 & 1.20\\
\bottomrule
\end{tabular}
\end{table}

\paragraph{Real LLM answers.} In Cora \citep{hpicora}, each entity comprises citation records of one paper. We use 62,423 within-entity pairs from 23 entities with at least 20 records. Qwen2.5-1.5B (base) \citep{qwen25model} answers 6.6\% wrongly. All orders reuse cached answers; similarity-first follows decreasing textual similarity.

Pivot loses less than random order, while similarity-first has the lowest point estimate (\cref{tab:realllm}). Pooled answer-error rates on noiseless query-tree edges are 2.2\% for similarity-first, 8.4\% for pivot and 8.6\% for random order, linking loss to both tree shape and edge accuracy. Pivot beats random order at all tested thresholds (approximately 1\%--20\% answer-error rates; \cref{tab:sweep}).

Redistributing each entity's errors uniformly, with error counts fixed and query orders matched, raises all three losses and removes similarity-first's advantage over pivot (\cref{tab:realllm}, rows~2--3). Thus error locations matter.

In six blocks of three or four similar entities, both false merges and false splits can occur; pivot has the lowest pooled loss. Appendix~\ref{app:realpilot} gives the prompting protocol, threshold sweep and detailed results for these settings.

\section{DISCUSSION}
\label{sec:discussion}

We analyzed error propagation in greedy, constraint-based decoding. Under uniformly random query orders, 2-CNF constraints do not increase first-order risk over raw answers, while shared-premise Horn rules can amplify errors. In record matching, one-entity first-order risk is determined by pairwise distances on the noiseless query tree. Under independent, equal-rate errors, pivot minimizes this risk among greedy orders. With one wrong answer, false merges affect only two entities. Simulations support the formulas and show pivot's advantage over random order at the tested higher error rates. The Cora results show that tree shape and error locations both matter: uniformly redistributing errors removes similarity-first's advantage over pivot. Future work could study interacting errors, especially false merges, across more models and datasets. The general Horn false-positive damage bound and false-merge growth law remain conjectures.

\clearpage 
\subsection*{AI Use Statement}
The research ideas, mathematical claims, experimental design, and interpretation of the results were developed by the authors. During the preparation of this paper, ChatGPT and Codex assisted with editing the prose, restructuring and improving portions of the code, and organizing data and relevant literature. We also consulted these tools to identify proof arguments that could benefit from clearer explanations. All AI-assisted work was independently reviewed and verified by the authors, including the proofs, code, data, references, and changes to the manuscript. The authors accept full responsibility for the final paper.

\bibliographystyle{plainnat}
\bibliography{refs}

\clearpage
\appendix
\thispagestyle{empty}
\onecolumn
\aistatstitle{How Much Does One Wrong Answer Cost?\\ Supplementary Materials}


\section{Proofs}\label{app:proofs}
\subsection{Setting and reduction}\label{app:proofs-setup}

We first express small-noise risk in terms of single-error damage, then prove the results for logical constraints, entity resolution and sorting. Numerical checks are in \cref{app:numerics}.

A \emph{code} is a finite nonempty set $\cC\subseteq\{0,1\}^n$.
Its elements are \emph{codewords}, and its $n$ positions are \emph{coordinates} (the binary queries).
The truth is a codeword $x^\star\in\cC$.
The answers are $y=x^\star\oplus Z$, where $Z_1,\dots,Z_n$ are independent and $Z_i\sim\Bern(\eps)$ with $\eps\in[0,\tfrac12]$.
We write $e_a$ for the $a$-th unit vector of $\{0,1\}^n$ and $\dH$ for the Hamming distance.
A decoder maps the answers $y$ to an estimate, possibly using a random seed $\pi$ independent of $Z$ to choose the query order or break ties; its output is $\hat x$.
Its risk is $R(\eps)=\E\,\dH(\hat x,x^\star)$, where the expectation is over $Z$ and $\pi$.
We use three decoders.
\begin{itemize}
\item The \emph{raw} decoder outputs $\hat x=y$. Hence $\Rraw(\eps)=\eps n$ exactly.
\item The \emph{ML} decoder outputs a codeword at minimum Hamming distance from $y$, with ties broken uniformly at random. For $\eps<\tfrac12$ this is maximum-likelihood decoding.
\item The \emph{greedy} decoder $G_\pi$ processes the coordinates in an order $\pi$. It starts from the empty set $L_0=\emptyset$ of accepted answers. At step $t$, with $i=\pi(t)$, it skips $i$ if all codewords consistent with $L_{t-1}$ have the same value at $i$ (the query is \emph{implied}); otherwise it accepts the answer and sets $L_t=L_{t-1}\cup\{(i,y_i)\}$. It outputs the unique codeword consistent with $L_n$. \Cref{lem:welldef} in \cref{app:proofs-horn} shows that this codeword exists and is unique. Unless stated otherwise, $\pi$ is uniform over the $n!$ orders and independent of $Z$.
\end{itemize}
For a codeword $x$ and a coordinate $a$ we define
\[
B_a(x)=\{b\neq a:\ x\oplus e_a\notin\cC,\ x\oplus e_a\oplus e_b\in\cC\},\qquad
m_a(x)=|B_a(x)|.
\]
If $m_a(x)\ge1$, an error at $a$ leaves the observed word at distance~1 from $x$ and from $m_a(x)$ alternative codewords. Each alternative differs from the original codeword in two positions and is therefore at distance~2 from $x$.
For a decoder with seed $\pi$, the \emph{single-error damage} at $a$ is the expected loss when the only wrong answer is at $a$:
\[
D(a)=\E_\pi\,\dH\big(\hat x(x^\star\oplus e_a,\pi),\,x^\star\big).
\]

Probabilistic unclean databases \citep{desa2019pud} formalize a noisy-channel model of database cleaning and seek the most likely intended database.
Their Theorems~7 and~8 relate this optimum to weighted minimum repairs under sufficiently high constraint weights and specified prior/channel assumptions; the update-repair result also requires attribute independence.
Union-bound analyses of ML decoding are surveyed by \citet{sason2006performance}; part~(i) of \cref{lem:reduction} gives an elementary first-order expansion for our finite code setting.
Part~(ii) applies the same reduction to greedy decoding.

\begin{proof}[Proof of \cref{lem:reduction}, remainder bound]
We show that every decoder that is exact on the noiseless input, that is, $\hat x(x^\star,\pi)=x^\star$ for every seed, satisfies for all $\eps\in[0,\tfrac12]$
\begin{equation}\label{eq:remainder}
\Big|\,R(\eps)-\eps\sum_{a=1}^n D(a)\,\Big|\ \le\ \eps^2\,n(n-1)\Big(\max_a D(a)+\frac n2\Big)\ \le\ \tfrac32\,\eps^2 n^2(n-1).
\end{equation}
Condition on the error set $S=\{i:\ Z_i=1\}$.
If $S=\emptyset$, then $y=x^\star$ and the loss is~0.
The event $S=\{a\}$ has probability $\eps(1-\eps)^{n-1}$.
On this event the conditional expected loss is $D(a)$, because $\pi$ is independent of $Z$.
If $|S|\ge2$, the loss lies in $[0,n]$, and a union bound over pairs of coordinates gives $\Prob(|S|\ge2)\le\binom n2\eps^2$.
Therefore
\[
R(\eps)-\eps\sum_a D(a)\ =\ -\eps\,\big[1-(1-\eps)^{n-1}\big]\sum_a D(a)\ +\ r,\qquad 0\le r\le\tfrac12\,n^2(n-1)\,\eps^2 .
\]
Bernoulli's inequality gives $0\le 1-(1-\eps)^{n-1}\le(n-1)\eps$.
So the first term lies in $[-\eps^2n(n-1)\max_aD(a),\,0]$, and the first inequality in \eqref{eq:remainder} follows.
The second inequality follows from $D(a)\le n$.
For ML, part~(i) below gives $D(a)<2$, which yields the tighter bound $\eps^2n(n-1)(2+n/2)$.

The same argument applies when coordinate $i$ has its own error rate $\eps_i\le\eps$.
Then $\Prob(S=\{a\})=\eps_a\prod_{b\neq a}(1-\eps_b)$, $0\le 1-\prod_{b\neq a}(1-\eps_b)\le(n-1)\eps$, and $\Prob(|S|\ge2)\le\binom n2\eps^2$.
Hence \eqref{eq:remainder} holds with $\eps\sum_aD(a)$ replaced by $\sum_a\eps_aD(a)$.
\end{proof}

\begin{proof}[Proof of \cref{lem:reduction}(i)]
We show that the ML decoder has $D_{\mathrm{ML}}(a)=1$ if $x^\star\oplus e_a\in\cC$ and $D_{\mathrm{ML}}(a)=2m_a(x^\star)/(1+m_a(x^\star))$ otherwise, so that
\[
\RML(\eps)=\eps\Big[N_1(x^\star)+\sum_{a:\,x^\star\oplus e_a\notin\cC}\frac{2m_a(x^\star)}{1+m_a(x^\star)}\Big]+O(\eps^2).
\]
Here $N_1(x^\star)=|\{a:\ x^\star\oplus e_a\in\cC\}|$ counts the codewords at Hamming distance one from the truth.
The ML decoder is exact on the noiseless input, because $x^\star$ is the only codeword at distance~0 from $y=x^\star$.
Let $y=x^\star\oplus e_a$.
If $y\in\cC$, the decoder returns $y$, which is at distance~0, and the loss is $\dH(y,x^\star)=1$.
Otherwise $\dH(y,\cC)=1$.
The codewords at distance~1 from $y$ are $x^\star=y\oplus e_a$ and the codewords of the form $y\oplus e_b=x^\star\oplus e_a\oplus e_b$ with $b\neq a$.
The latter are exactly the $m_a(x^\star)$ words indexed by $B_a(x^\star)$, and each of them is at distance~2 from $x^\star$.
Uniform tie-breaking among these $1+m_a(x^\star)$ words gives the expected loss $2m_a(x^\star)/(1+m_a(x^\star))$.
The displayed form of $\RML$ follows from \eqref{eq:remainder}.
If $m_a(x^\star)\le1$ for every $a$, as on 2-CNF theories (see the proof of \cref{prop:ml2cnf}), the bracket equals $N_1(x^\star)+\#\{a:\ m_a(x^\star)=1\}$.
In words, each codeword at distance~1 counts~1, and each two-way tie counts $2\times\frac12$.
\end{proof}

\begin{proof}[Proof of \cref{lem:reduction}(ii)]
We show that $\Rgr(\eps)=\eps\sum_a\E_\pi D_\pi(a)+O(\eps^2)$, with the remainder bound \eqref{eq:remainder}, where $D_\pi(a)=\dH(G_\pi(x^\star\oplus e_a),x^\star)$ is the damage of one wrong answer at $a$, whether this answer is accepted or not.
The output of $G_\pi$ is a function of $y$ and of the order $\pi$, and $\pi$ is independent of $Z$.
The decoder is exact on the noiseless input.
Indeed, if $y=x^\star$, every accepted pair $(i,y_i)$ agrees with $x^\star$; so $x^\star$ is consistent with $L_n$, and it is therefore the unique codeword consistent with $L_n$ (\cref{lem:welldef}).
So \eqref{eq:remainder} applies with $D(a)=\E_\pi D_\pi(a)$.
Let $\bar D=\frac1n\sum_a\E_\pi D_\pi(a)$ be the mean single-error damage over positions and orders.
Then $\Rgr(\eps)=\eps n\bar D+O(\eps^2)$.
Since $\Rraw(\eps)=\eps n$, the amplification of the greedy decoder over the raw decoder is $\bar D$.
\end{proof}

\subsection{Logical constraints}\label{app:proofs-horn}

Let $\Phi$ be a satisfiable propositional formula, that is, a conjunction of clauses, over a finite set $V$ of atoms with $|V|=n$.
The code is the set $\cC=\operatorname{Mod}(\Phi)\subseteq\{0,1\}^V$ of models (satisfying assignments) of $\Phi$.
Every nonempty code $\cC\subseteq\{0,1\}^n$ has this form: take for $\Phi$ the conjunction, over the words $w\notin\cC$, of the clause that is false exactly at $w$.
The truth is a model $\sigma^\star\in\cC$.
For an atom $x$, the \emph{true literal} $t_x$ is $x$ if $\sigma^\star(x)=1$ and $\neg x$ otherwise; the \emph{false literal} $\bar t_x$ is the other literal on $x$.
The answer on $x$ is $y_x=\sigma^\star(x)\oplus Z_x$, where the $Z_x$ are i.i.d.\ $\Bern(\eps)$.
The \emph{claim literal} of $x$ is $x$ if $y_x=1$ and $\neg x$ if $y_x=0$.
A set $L$ of literals records accepted answers; requiring all its members to hold is equivalent to taking their conjunction.
We write $L\models_\Phi m$ if $\Phi\wedge L\models m$, that is, if every model of $\Phi$ that satisfies all literals of $L$ also satisfies the literal $m$.

For an order $\pi$ of $V$, the greedy decoder $G_\pi$ (skip if implied, accept otherwise) starts with $L_0=\emptyset$.
At step $t$, with $x=\pi(t)$, it sets $L_t=L_{t-1}$ if $L_{t-1}\models_\Phi x$ or $L_{t-1}\models_\Phi\neg x$ (the atom is \emph{skipped}); otherwise it sets $L_t=L_{t-1}\cup\{\text{claim literal of }x\}$ (the claim is \emph{accepted}).
Its output is the unique model of $\Phi\wedge L_n$.
On the code $\operatorname{Mod}(\Phi)$ this is the greedy decoder of \cref{app:proofs-setup}, because ``all models of $\Phi$ that satisfy $L$ have the same value at $x$'' means ``$L\models_\Phi x$ or $L\models_\Phi\neg x$'' whenever $\Phi\wedge L$ is satisfiable.
The loss is the Hamming distance, and $\Rgr(\eps)=\E_{Z,\pi}\,\dH(G_\pi(y),\sigma^\star)$ with $\pi$ uniform and independent of $Z$.
The raw decoder outputs $y$ and has $\Rraw(\eps)=\eps n$ exactly.

\begin{lemma}[Well-definedness; equivalence with accept-if-consistent]\label{lem:welldef}
Let $\Phi$ be a satisfiable formula over $V$, let $\pi$ be an order of $V$, and let $y\in\{0,1\}^V$ be arbitrary answers.
\begin{enumerate}
\item[(i)] $\Phi\wedge L_t$ is satisfiable for every $t$.
\item[(ii)] $\Phi\wedge L_n$ has exactly one model, so the output $G_\pi(y)$ is well defined.
\item[(iii)] $G_\pi(y)$ equals the output of the accept-if-consistent decoder. This decoder processes the atoms in the same order, starts with $L'_0=\emptyset$, adds the claim literal $\ell$ of the current atom to $L'$ if and only if $\Phi\wedge L'\wedge\ell$ is satisfiable, and outputs the unique model of $\Phi\wedge L'_n$.
\end{enumerate}
\end{lemma}

\begin{proof}
(i) We use induction on $t$. The formula $\Phi\wedge L_0=\Phi$ is satisfiable by assumption. A literal $m$ is added at step $t$ only if $L_{t-1}\not\models_\Phi\neg m$, that is, only if $\Phi\wedge L_{t-1}\wedge m$ is satisfiable.

(ii) At its own step $t$, each atom $x$ either has its value entailed by $\Phi\wedge L_{t-1}$ or receives a literal in $L_t$.
Entailment is monotone in $L$, and $L_{t-1}\subseteq L_n$.
So $\Phi\wedge L_n$ entails a value for every atom.
By (i) it has a model, and this model is unique because every atom's value is entailed.

(iii) We show by induction on $t$ that $\Phi\wedge L_t$ and $\Phi\wedge L'_t$ have the same set of models.
This holds for $t=0$.
Suppose it holds for $t-1$; then both runs entail the same literals before step $t$.
Let $\ell$ be the claim literal of the atom processed at step $t$.
If neither value of the atom is entailed, then $\neg\ell$ is not entailed, so $\ell$ is consistent, and both decoders add $\ell$.
If $\ell$ is entailed, the greedy decoder skips the atom and the accept-if-consistent decoder adds $\ell$; adding an entailed literal does not change the set of models.
If $\neg\ell$ is entailed, $\ell$ is inconsistent, and both decoders leave their sets unchanged.
In all cases the two sets of models remain equal.
At $t=n$ this set is a single model by (ii), so the two outputs coincide.
\end{proof}

Since every nonempty code is $\operatorname{Mod}(\Phi)$ for some satisfiable $\Phi$, \cref{lem:welldef} also shows that the greedy decoder of \cref{app:proofs-setup} is well defined on every code.

The \emph{backbone} is $B=\{x\in V:\ \Phi\models t_x\}$, the set of atoms that take the same value in every model.
For $a\notin B$ we put
\[
K_a=\{x\in V:\ \Phi\wedge\bar t_a\models\bar t_x\},\qquad \sigma^{(a)}=\sigma^\star\oplus\ind_{K_a},
\]
and for $a\in B$ we put $K_a=\emptyset$.
When $a\notin B$, the set $K_a$ contains $a$, and we call it the \emph{flip set} of $a$.
In this case $\Phi\wedge\bar t_a$ is satisfiable, so it cannot entail both $t_x$ and $\bar t_x$, and $\sigma^{(a)}$ makes $\bar t_x$ true for every $x\in K_a$.
We say that $(\Phi,\sigma^\star)$ has the \emph{single-flip property at $a$}, written (P$_a$), if $\sigma^{(a)}$ is a model of $\Phi$. This means that flipping the forced atoms is enough to accommodate the wrong answer, without changing any other atom.
Under (P$_a$), $\sigma^{(a)}$ is the unique model of $\Phi\wedge\bar t_a$ that agrees with $\sigma^\star$ outside $K_a$, because every model of $\Phi\wedge\bar t_a$ satisfies $\bar t_x$ for all $x\in K_a$.
The single-error damage of the greedy decoder at $a$ is $D_\pi(a)=\dH(G_\pi(y),\sigma^\star)$ for the answers $y=\sigma^\star\oplus e_a$, in which only the answer on $a$ is wrong.

\begin{proof}[Proof of \cref{thm:singleflip}]
Suppose that exactly one answer is wrong, at the atom $a$. Then $D_\pi(a)=0$ for every order $\pi$ when $a\in B$.
If $a\notin B$ and (P$_a$) holds, then for every order $\pi$
\[
D_\pi(a)=|K_a|\cdot\ind\{a\text{ precedes every element of }K_a\setminus\{a\}\text{ in }\pi\},
\]
so that $\E_\pi D_\pi(a)=1$ for the uniform order, and indeed for every exchangeable random order.

\emph{Case $a\in B$.}
Here $\emptyset\models_\Phi t_a$, so the atom $a$ is always skipped and its wrong claim $\bar t_a$ is never accepted.
All other claims are true literals.
Hence all accepted literals are true, $\sigma^\star$ is a model of $\Phi\wedge L_n$, and by \cref{lem:welldef}(ii) the output is $\sigma^\star$.

\emph{Case $a\notin B$ under (P$_a$).}
Let $\tau$ be the step at which $a$ is processed.
For $t<\tau$, every accepted literal is the claim of an atom with a correct answer, so $L_t$ consists of true literals.

\emph{Step 1 (when $t_a$ is entailed).}
For every set $S$ of true literals, $S\models_\Phi t_a$ if and only if $S$ contains a literal on some atom of $K_a$.
Indeed, $S\models_\Phi t_a$ if and only if $\Phi\wedge S\wedge\bar t_a$ is unsatisfiable.
If $S$ contains no literal on an atom of $K_a$, then $\sigma^{(a)}$ satisfies $S$, because it agrees with $\sigma^\star$ outside $K_a$; it also satisfies $\bar t_a$, and it satisfies $\Phi$ by (P$_a$).
So $\Phi\wedge S\wedge\bar t_a$ is satisfiable, and $S\not\models_\Phi t_a$.
If $S$ contains $t_x$ with $x\in K_a$, then $\Phi\wedge\bar t_a\models\bar t_x$ contradicts $t_x$, so $\Phi\wedge S\wedge\bar t_a$ is unsatisfiable.

\emph{Step 2 (when the wrong claim is accepted).}
A literal enters $L$ only at the step of its own atom.
So if $L_{\tau-1}$ contains a literal on an atom of $K_a$, this atom lies in $K_a\setminus\{a\}$ and precedes $a$.
Conversely, suppose that some $x\in K_a\setminus\{a\}$ is processed at a step $t<\tau$.
Either $t_x$ was added, and then $L_{\tau-1}$ contains a literal on $K_a$; or $x$ was skipped.
In the second case $L_{t-1}$ entails $t_x$ or $\bar t_x$; it cannot entail $\bar t_x$, because $\sigma^\star$ is a model of $\Phi\wedge L_{t-1}$; so $L_{t-1}\models_\Phi t_x$.
Then $\Phi\wedge L_{t-1}\wedge\bar t_a$ entails both $t_x$ and $\bar t_x$, so it is unsatisfiable, that is, $L_{t-1}\models_\Phi t_a$.
By Step~1, $L_{t-1}\subseteq L_{\tau-1}$ contains a literal on $K_a$.
Together with Step~1 this shows: at step $\tau$, $L_{\tau-1}\models_\Phi t_a$ if and only if some element of $K_a\setminus\{a\}$ precedes $a$.
The wrong literal $\bar t_a$ is never entailed at step $\tau$, because $\sigma^\star$ is a model of $\Phi\wedge L_{\tau-1}$.
Hence $\bar t_a$ is accepted if and only if $a$ is the first atom of $K_a$ in $\pi$; otherwise $a$ is skipped.

\emph{Step 3 (rejected case).}
If $\bar t_a$ is not accepted, every literal ever added is true, so the output is $\sigma^\star$ and $D_\pi(a)=0$.

\emph{Step 4 (accepted case).}
Now $L_n=\{\bar t_a\}\cup S$, where $S$ consists of true literals of atoms other than $a$.
No atom of $K_a\setminus\{a\}$ contributes to $S$.
Indeed, such an atom $x$ comes after $a$; at its step $L$ contains $\bar t_a$, and $\{\bar t_a\}\models_\Phi\bar t_x$ by the definition of $K_a$; so $x$ is skipped.
Hence $\sigma^{(a)}$ satisfies $S$ (it agrees with $\sigma^\star$ outside $K_a$), satisfies $\bar t_a$, and satisfies $\Phi$ by (P$_a$).
By \cref{lem:welldef}(ii) the output is $\sigma^{(a)}$, and $D_\pi(a)=|K_a|$.

\emph{Expectation.}
For the uniform order, and for every exchangeable random order, each of the $|K_a|$ atoms of $K_a$ is equally likely to come first among them.
So $a$ comes first with probability $1/|K_a|$, and $\E_\pi D_\pi(a)=|K_a|\cdot|K_a|^{-1}=1$.
\end{proof}

\begin{proof}[Proof of \cref{prop:twocnf}]
We show that if every clause of $\Phi$ has at most two literals, then (P$_a$) holds for every truth $\sigma^\star\in\cC$ and every atom $a\notin B$.
Let $E$ be the set of literals entailed by $\Phi\wedge\bar t_a$.
Since $\Phi\wedge\bar t_a$ is satisfiable, $E$ contains no complementary pair.
By definition, for $x\in K_a$ the literal $\bar t_x$ lies in $E$, and $\sigma^{(a)}$ makes it true.
Take a clause $p\vee q$ of $\Phi$ and suppose that $\sigma^{(a)}$ falsifies both $p$ and $q$.
If neither atom lies in $K_a$, then $\sigma^{(a)}$ agrees with $\sigma^\star$ on both atoms, and $\sigma^\star$ satisfies the clause; this is a contradiction.
So we may assume that $\operatorname{atom}(p)\in K_a$, where $\operatorname{atom}(p)$ denotes the atom of the literal $p$.
Then $\neg p$ is true under $\sigma^{(a)}$, and it equals $\bar t_{\operatorname{atom}(p)}$, which lies in $E$.
From $\Phi\models p\vee q$ we get $q\in E$.
If $\operatorname{atom}(q)\in K_a$, the unique literal of $E$ on this atom is $\bar t_{\operatorname{atom}(q)}$, which $\sigma^{(a)}$ makes true; so $q$ is true, a contradiction.
If $\operatorname{atom}(q)\notin K_a$, then $\bar t_{\operatorname{atom}(q)}\notin E$, so $q=t_{\operatorname{atom}(q)}$; this literal is true under $\sigma^\star$ and hence under $\sigma^{(a)}$, again a contradiction.
Finally, take a unit clause $(p)$.
Then $p\in E$, and $p=t_{\operatorname{atom}(p)}$ because $\sigma^\star\models p$.
If $\operatorname{atom}(p)$ were in $K_a$, then $E$ would contain both $p$ and $\bar t_{\operatorname{atom}(p)}=\neg p$; so $\operatorname{atom}(p)\notin K_a$, and $\sigma^{(a)}$ satisfies $p$.
\end{proof}

\paragraph{Implication theories.}
For a directed graph $G=(V,A)$ let $\Phi_G=\bigwedge_{(u,v)\in A}(\neg u\vee v)$.
Identify an assignment with its set of true atoms.
The models of $\Phi_G$ are then the \emph{closed} sets: the sets $M$ such that $u\in M$ and $(u,v)\in A$ imply $v\in M$ (the up-closed sets of the preorder generated by $G$).
For $X\subseteq V$, let $\operatorname{cl}(X)$ be the smallest closed set that contains $X$.
Since $\emptyset$ and $V$ are models, the backbone is empty.
Let $M^\star$ be the true set, let $\operatorname{reach}(a)$ be the set of vertices reachable from $a$ (including $a$), and let $\operatorname{anc}(a)$ be the set of vertices from which $a$ is reachable (including $a$).
Then:
\begin{itemize}
\item for a false positive ($a\notin M^\star$), $K_a=\operatorname{reach}(a)\setminus M^\star$ and $\sigma^{(a)}=M^\star\cup\operatorname{reach}(a)$; hence $|K_a|=c_a:=|\operatorname{cl}(M^\star\cup\{a\})\setminus M^\star|$, the \emph{cone} of $a$;
\item for a false negative ($a\in M^\star$), $K_a=\operatorname{anc}(a)\cap M^\star$ and $\sigma^{(a)}=M^\star\setminus\operatorname{anc}(a)$.
\end{itemize}
Indeed, the models of $\Phi_G\wedge a$ are the closed sets that contain $\operatorname{reach}(a)$; the smallest is $\operatorname{reach}(a)$ and the largest is $V$.
So the literals entailed by $\Phi_G\wedge a$ are exactly the positive literals on $\operatorname{reach}(a)$.
For $a\notin M^\star$ we have $\bar t_a=a$, and $\bar t_x$ is the positive literal $x$ exactly when $x\notin M^\star$; this gives $K_a=\operatorname{reach}(a)\setminus M^\star$.
Since $M^\star$ is closed, $M^\star\cup\operatorname{reach}(a)$ is closed and equals $\operatorname{cl}(M^\star\cup\{a\})$.
Dually, the models of $\Phi_G\wedge\neg a$ are the closed sets contained in $V\setminus\operatorname{anc}(a)$, which is itself closed; the smallest is $\emptyset$ and the largest is $V\setminus\operatorname{anc}(a)$.
So the literals entailed by $\Phi_G\wedge\neg a$ are exactly the negative literals on $\operatorname{anc}(a)$.
For $a\in M^\star$ we have $\bar t_a=\neg a$, and $\bar t_x=\neg x$ exactly when $x\in M^\star$; this gives $K_a=\operatorname{anc}(a)\cap M^\star$.

\begin{proof}[Proof of \cref{cor:conservation}]
We show that for a satisfiable 2-CNF $\Phi$ with backbone $B$, every truth $\sigma^\star\in\cC$ and every $\eps\in[0,\tfrac12]$,
\[
-\eps^2n(n-1)\ \le\ \Rgr(\eps)-\eps\,(n-|B|)\ \le\ \tfrac12\eps^2n^2(n-1).
\]
We split on the error set $S=\{x:\ Z_x=1\}$.
With no error the output is $\sigma^\star$ (\cref{lem:welldef}(ii)), and the loss is~0.
With exactly one error, at $a$ (probability $\eps(1-\eps)^{n-1}$), \cref{thm:singleflip,prop:twocnf} give $\E_\pi D_\pi(a)=\ind\{a\notin B\}$.
Summed over $a$, these events contribute $\eps(1-\eps)^{n-1}(n-|B|)$.
With two or more errors the loss lies in $[0,n]$ and the probability is at most $\binom n2\eps^2$; these events contribute a term $r\in[0,\tfrac12n^2(n-1)\eps^2]$.
Hence $\Rgr(\eps)-\eps(n-|B|)=-\eps(n-|B|)\big[1-(1-\eps)^{n-1}\big]+r$.
Finally, $0\le\eps(n-|B|)[1-(1-\eps)^{n-1}]\le\eps^2(n-1)n$ by Bernoulli's inequality, which gives both bounds.

For an implication theory $\Phi_G$ the backbone is empty, so $\Rgr(\eps)=\eps n+O(\eps^2)=\Rraw(\eps)+O(\eps^2)$ for every truth.
The per-atom statement $\E_\pi D_\pi(a)=1$ holds for each atom separately, so the conclusion holds for false-positive and false-negative errors alike.
\end{proof}

\begin{remark}[The random order is essential]\label{rem:order}
For a fixed order the single-error damage takes only the values $0$ and $|K_a|$.
Under a random order, the wrong claim is accepted with probability $1/|K_a|$, and then it damages exactly $|K_a|$ atoms.
A fixed order can make every error of one type maximally damaging.
Consider the chain $v_1\to v_2\to\dots\to v_n$ (clauses $\neg v_i\vee v_{i+1}$), whose models are the up-sets $\{v_j,\dots,v_n\}$ and $\emptyset$.
For a false positive at $a$, the flip set $K_a=\operatorname{reach}(a)\setminus M^\star$ consists of $a$ and the false atoms after it on the chain, and the wrong claim is accepted if and only if $a$ is asked before every other atom of $K_a$.
So the topological order $v_1,v_2,\dots,v_n$ accepts every false positive and lets it corrupt every false atom reachable from it, while the reverse order $v_n,\dots,v_1$ rejects every false positive with $|K_a|\ge2$.
For example, with the chain $a\to b$, the truth ``both false'' and the wrong answer $a=1$, the order $(a,b)$ accepts $a$, then finds $b$ implied true, and has damage~2; the order $(b,a)$ accepts $b=0$, then finds $a$ implied false, and has damage~0.
For false negatives the roles of the two orders are exchanged: $K_a=\operatorname{anc}(a)\cap M^\star$ consists of $a$ and the true atoms before it, and the reverse topological order is the damaging one.
\end{remark}

\begin{proof}[Proof of \cref{prop:ml2cnf}]
We show that on a satisfiable 2-CNF with truth $\sigma^\star$, a single wrong answer at an atom $a\notin B$ costs the ML decoder exactly~1 in expectation if $|K_a|\le2$ and~0 if $|K_a|\ge3$, and that a single wrong answer at an atom $a\in B$ costs~0.
Consequently $\RML(\eps)=\eps\,|\{a\notin B:\ |K_a|\le2\}|+O(\eps^2)$.

We use \cref{lem:reduction}(i).
Let $y=\sigma^\star\oplus e_a$, and let $m_a$ be the number of codewords of the form $y\oplus e_b$ with $b\neq a$ when $y\notin\cC$.
Then $D_{\mathrm{ML}}(a)=1$ if $y\in\cC$ and $D_{\mathrm{ML}}(a)=2m_a/(1+m_a)$ otherwise.
Every word that differs from $\sigma^\star$ at $a$ satisfies $\bar t_a$.

\emph{Case $a\in B$.}
No model satisfies $\bar t_a$.
So $y\notin\cC$, no word $y\oplus e_b$ is a codeword, $m_a=0$ and $D_{\mathrm{ML}}(a)=0$.

\emph{Case $a\notin B$.}
By \cref{prop:twocnf}, (P$_a$) holds.
Every model that satisfies $\bar t_a$ satisfies $\bar t_x$ for all $x\in K_a$, so it differs from $\sigma^\star$ on all of $K_a$.
If $y\in\cC$, then $y$ differs from $\sigma^\star$ only at $a$, so $K_a=\{a\}$.
Conversely, if $K_a=\{a\}$, then $y=\sigma^{(a)}$ is a model by (P$_a$).
So $y\in\cC$ if and only if $|K_a|=1$, and in this case $D_{\mathrm{ML}}(a)=1$.
Now let $|K_a|\ge2$, so that $y\notin\cC$.
A codeword $y\oplus e_b=\sigma^\star\oplus e_a\oplus e_b$ with $b\neq a$ differs from $\sigma^\star$ exactly on $\{a,b\}$ and satisfies $\bar t_a$; hence $K_a\subseteq\{a,b\}$, that is, $K_a=\{a,b\}$.
Conversely, if $K_a=\{a,b\}$, then $\sigma^\star\oplus e_a\oplus e_b=\sigma^{(a)}$ is a model by (P$_a$).
So $m_a=1$ if $|K_a|=2$, and $m_a=0$ if $|K_a|\ge3$.
With uniform tie-breaking this gives $D_{\mathrm{ML}}(a)=2\cdot\frac12=1$ and $D_{\mathrm{ML}}(a)=0$, respectively.

Summing over the atoms and applying \eqref{eq:remainder} gives the formula for $\RML$.
In particular $m_a\le1$ for every atom, as used in the proof of \cref{lem:reduction}(i).
Comparing with \cref{cor:conservation}, $\Rgr(\eps)-\RML(\eps)=\eps\,|\{a\notin B:\ |K_a|\ge3\}|+O(\eps^2)$.
On an implication theory ($B=\emptyset$), the greedy and raw decoders therefore have the same first-order risk, while the greedy risk exceeds the ML risk by $\eps\,|\{a:\ |K_a|\ge3\}|$ at first order.
\end{proof}

\paragraph{Horn hubs.}
For integers $k\ge1$ and $s\ge2$, the Horn hub $H_{k,s}$ has the atom $a$ and, for each $i=1,\dots,k$, the atoms $b_{i,1},\dots,b_{i,s-1},c_i$; so $n=1+ks$.
Its rules are $b_{i,1}\wedge\dots\wedge b_{i,s-1}\wedge a\to c_i$ for $i=1,\dots,k$, each with a body of $s$ atoms.
The corresponding clauses are $C_i=\neg b_{i,1}\vee\dots\vee\neg b_{i,s-1}\vee\neg a\vee c_i$.
The truth is $M^\star=\{b_{i,j}:\ i\le k,\ j\le s-1\}$, the assignment in which exactly the atoms $b_{i,j}$ are true; it is a model because $a$ is false.
Block $i$ is $S_i=\{b_{i,1},\dots,b_{i,s-1},c_i\}$, and $T_i=\{b_{i,1},\dots,b_{i,s-1},\neg c_i\}$ is the set of its true literals.
The cone of $a$ is $c_a=k+1$: the closure of $M^\star\cup\{a\}$ adds $a$ and $c_1,\dots,c_k$.
For $s=2$ the rules have the form of the rule $\mathrm{isa}(x,y)\wedge\mathrm{has}(y,z)\to\mathrm{has}(x,z)$ with a shared premise.

\begin{proof}[Proof of \cref{thm:hub}]
For uniform $\pi$, a single false positive on $a$ (all other answers correct) has
\[
\E_\pi D_\pi(a)=(k+1)\int_0^1(1-u^s)^k\,du=(k+1)\,\frac{\Gamma(1+1/s)\,\Gamma(k+1)}{\Gamma(k+1+1/s)}=\Gamma\big(1+\tfrac1s\big)\,k^{1-1/s}\,\big(1+O(1/k)\big),
\]
where $\Gamma$ is the Gamma function.
For $s=2$ this equals $(k+1)\,4^k(k!)^2/(2k+1)!\sim\frac{\sqrt\pi}{2}\sqrt k$.
Every other single error has damage exactly~1 for every $\pi$.
Hence
\[
\Rgr(\eps)=\eps\,\big[ks+\E_\pi D_\pi(a)\big]+O(\eps^2)=\Rraw(\eps)+\eps\,\big[\E_\pi D_\pi(a)-1\big]+O(\eps^2).
\]

\emph{Satisfiability criterion.}
Let $L'$ be a set of literals without complementary pairs.
We claim that $\Phi\wedge L'$ is unsatisfiable if and only if $a\in L'$ and $T_i\subseteq L'$ for some $i$, that is, $L'$ falsifies every literal of some clause $C_i$ other than $\neg a$.
The ``if'' direction is clear.
For the ``only if'' direction, suppose first that $a\notin L'$.
Setting $a$ false and every other atom as $L'$ prescribes (free atoms arbitrarily) satisfies every clause, since each clause contains $\neg a$.
Suppose next that $a\in L'$ but $T_i\not\subseteq L'$ for every $i$.
Then each clause $C_i$ has a literal on $S_i$ that is either true under $L'$ or lies on an atom that $L'$ leaves free.
Set each free atom of $S_i$ so as to satisfy $C_i$.
The atoms of $S_i$ occur only in $C_i$, so these choices do not interact, and $\Phi\wedge L'$ is satisfiable.

\emph{The run with the wrong claim $a$.}
All other claims are the true literals, the members of the sets $T_i$.
Let $\tau$ be the step of $a$.
\begin{itemize}
\item Before $\tau$, the atom $a$ has no literal in $L$. For an atom $x\neq a$ processed before $\tau$, neither $L\cup\{x\}$ nor $L\cup\{\neg x\}$ contains $a$, so both are consistent with $\Phi$ by the criterion. Hence nothing is entailed and every claim is accepted. So $L_{\tau-1}$ is the set of true literals of the atoms processed before $a$.
\item At $\tau$, the literal $a$ is not entailed, because $L_{\tau-1}\cup\{\neg a\}$ does not contain $a$ and is consistent. The literal $\neg a$ is entailed if and only if $\Phi\wedge L_{\tau-1}\wedge a$ is unsatisfiable, that is, if and only if $T_i\subseteq L_{\tau-1}$ for some $i$. This means that some block lies entirely before $a$ in $\pi$. In that case the wrong claim is skipped, all accepted literals are true, and the output is $\sigma^\star$ (damage~0).
\item Otherwise the wrong claim $a$ is accepted. Consider block $i$ and an atom $x\in S_i$ processed after $a$, and let $L$ be the set of accepted literals at the step of $x$. Its true literal $t_x$ is not entailed. Indeed, unsatisfiability of $\Phi\wedge L\wedge\bar t_x$ would require $T_j\subseteq L\cup\{\bar t_x\}$ for some $j$. The case $j=i$ is impossible because $t_x\in T_i$ but $t_x\notin L\cup\{\bar t_x\}$. The case $j\neq i$ would give $T_j\subseteq L$, which contradicts the satisfiability of $\Phi\wedge L$ (\cref{lem:welldef}(i)). By the same argument, the false literal $\bar t_x$ is entailed if and only if $\Phi\wedge L\wedge t_x$ is unsatisfiable, that is, if and only if $T_i\setminus\{t_x\}\subseteq L$.
\end{itemize}
Now follow block $i$ along the order $\pi$.
Every atom of $S_i$ other than the last one in $\pi$ is accepted with its true literal.
For atoms before $a$ this follows from the first bullet.
For an atom $x$ after $a$ that is not the last atom of $S_i$, the literal of the last atom is not yet in $L$, so $T_i\setminus\{t_x\}\not\subseteq L$, and neither literal of $x$ is entailed.
The last atom of $S_i$ comes after $a$, because $S_i$ does not lie entirely before $a$.
At its step the other $s-1$ literals of $T_i$ are in $L$, so its false literal is entailed, and the output gives it its false value.
So each block contributes damage exactly~1, and together with $a$ the damage is $k+1$.

\emph{Probability of acceptance.}
Attach independent $\mathrm{Unif}[0,1]$ times to all atoms; sorting the times gives a uniform order.
Conditionally on the time $u$ of $a$, block $i$ lies entirely before $a$ with probability $u^s$, independently over the blocks.
Hence the wrong claim is accepted with probability $\int_0^1(1-u^s)^k\,du$, and $\E_\pi D_\pi(a)=(k+1)\int_0^1(1-u^s)^k\,du$.
The substitution $t=u^s$ turns the integral into $\frac1s\,\mathrm B(\frac1s,k+1)$, where $\mathrm B(\alpha,\beta)=\int_0^1t^{\alpha-1}(1-t)^{\beta-1}dt=\Gamma(\alpha)\Gamma(\beta)/\Gamma(\alpha+\beta)$ is the Beta function.
Since $\frac1s\Gamma(\frac1s)=\Gamma(1+\frac1s)$, the integral equals $\Gamma(1+\frac1s)\Gamma(k+1)/\Gamma(k+1+\frac1s)$.
For $s=2$, the identity $\Gamma(k+\frac32)=\frac{(2k+2)!}{4^{k+1}(k+1)!}\sqrt\pi$ and $\Gamma(\frac32)=\frac{\sqrt\pi}2$ give
\[
\frac{\Gamma(\frac32)\,k!}{\Gamma(k+\frac32)}=\frac{2\cdot4^k\,k!\,(k+1)!}{(2k+2)!}=\frac{4^k(k!)^2}{(2k+1)!}.
\]
The asymptotic form follows from $\Gamma(k+1)/\Gamma(k+1+\alpha)=k^{-\alpha}(1+O(1/k))$ for fixed $\alpha\in(0,1]$, together with $k+1=k(1+1/k)$.
For $s=2$ it reads $\E_\pi D_\pi(a)\sim\Gamma(\frac32)\sqrt k=\frac{\sqrt\pi}2\sqrt k$.

\emph{Other single errors.}
The atom $a$ is false in $M^\star$, so its only possible error is the false positive treated above.
Consider a false negative on some $b_{i,j}$ or a false positive on some $c_i$, with all other answers correct.
The answer vector $y$ then has $a$ false, so it satisfies every clause and is itself a model.
Hence every claim is consistent with every set of accepted claims, since all of them are literals of the model $y$.
So the negation of a claim is never entailed, and every atom either has its claim accepted or has its claim entailed.
Thus $y$ is a model of $\Phi\wedge L_n$, and by \cref{lem:welldef}(ii) the output is $y$, with damage $\dH(y,\sigma^\star)=1$, for every $\pi$.

\emph{Summation.}
There are $ks$ atoms other than $a$, each with single-error damage~1, and $n=ks+1$.
By \cref{lem:reduction}(ii), $\Rgr(\eps)=\eps[ks+\E_\pi D_\pi(a)]+O(\eps^2)=\eps n+\eps[\E_\pi D_\pi(a)-1]+O(\eps^2)$.
Since $\max_x\E_\pi D_\pi(x)\le k+1$, the remainder bound \eqref{eq:remainder} gives $|\Rgr(\eps)-\eps[ks+\E_\pi D_\pi(a)]|\le\eps^2n(n-1)(k+1+\frac n2)$.
The excess $\E_\pi D_\pi(a)-1$ is $\Theta(k^{1-1/s})$, so the greedy risk exceeds the raw risk by $\eps\,\Theta(k^{1-1/s})$ at first order.
\end{proof}

\begin{remark}[Where the single-flip property fails]\label{rem:singleflipfails}
In $H_{k,s}$ the formula $\Phi\wedge a$ entails nothing new.
Indeed, by the satisfiability criterion $\{a,\ell\}$ is consistent for every literal $\ell$ on another atom, because each $T_i$ has $s\ge2$ literals.
So $K_a=\{a\}$ and $\sigma^{(a)}=M^\star\cup\{a\}$, which violates every clause $C_i$ and is not a model.
The wrong literal alone does not imply additional incorrect values. These values are implied later, when the wrong literal is combined with correct literals. Acceptance is determined by the $s$-element blocks rather than by the flip set.
The single-flip property can also fail for false negatives.
Change the signs of the literals on $a$ and on $c_i$ in every clause; this gives the rules $b_{i,1}\wedge\dots\wedge b_{i,s-1}\wedge c_i\to a$ (for $s=2$, the rules $b_i\wedge c_i\to a$).
Take the truth in which all atoms are true, and a false negative on $a$.
The unsatisfiability criterion becomes: $\neg a\in L'$ and some block has all its true literals in $L'$.
The proof above, with the literals of each block exchanged, gives the same damage law.
For comparison with ML, when $k\ge2$ the word $M^\star\cup\{a\}$ has $M^\star$ as its unique nearest codeword, because flipping a single atom of a block repairs only one of the $k$ violated clauses.
So the ML single-error damage at $a$ is~0.
\end{remark}

\paragraph{General Horn theories.}
A \emph{definite Horn theory} is a conjunction of rules $b_1\wedge\dots\wedge b_r\to h$ over atoms, where $\{b_1,\dots,b_r\}$ is the \emph{body} and $h$ the \emph{head} of the rule.
Identify a model with its set of true atoms.
The models are then the \emph{closed} sets, that is, the sets that contain the head of every rule whose body they contain.
For $X\subseteq V$, $\operatorname{cl}(X)$ is the smallest closed set that contains $X$, obtained by repeatedly adding conclusions whose premises are already present.
For a closed truth $M^\star$ and an atom $a\notin M^\star$, the \emph{cone} of $a$ is $c_a=|\operatorname{cl}(M^\star\cup\{a\})\setminus M^\star|$; this extends the cone of implication theories.

\begin{conjecture}[Body-size law]\label{conj:bodysize}
For every $s\ge2$ there is a constant $C_s<\infty$ with the following property.
Let $\Phi$ be a definite Horn theory whose rule bodies have at most $s$ atoms, let $M^\star$ be a closed set, and let $a\notin M^\star$.
Then a single false positive at $a$ (all other answers correct) has, for uniform $\pi$,
\[
\E_\pi D_\pi(a)\ \le\ C_s\,c_a^{\,1-1/s}.
\]
\end{conjecture}

By \cref{thm:hub} the exponent cannot be lowered: the hub $H_{k,s}$ has $c_a=k+1$ and $\E_\pi D_\pi(a)\sim\Gamma(1+\frac1s)\,k^{1-1/s}$.
For $s=1$, that is, for implication theories, the statement holds with $C_1=1$ and with equality: there $|K_a|=c_a$ for a false positive, and $\E_\pi D_\pi(a)=1$ by \cref{thm:singleflip,prop:twocnf}.
A corresponding bound for false negatives remains open; it requires a way to measure the atoms affected when a true premise is removed.
In the simulations of \cref{app:numerics}, which also test \cref{conj:bodysize}, false-negative damages reach $1.48$.

\subsection{Entity resolution}\label{app:proofs-er}

The items are $[m]=\{1,\dots,m\}$.
The code is the set of partitions of $[m]$, written as vectors of $\binom m2$ pair indicators $x_{uv}=\ind\{u\sim v\}$, where $u\sim v$ means that $u$ and $v$ lie in the same block.
The constraints are transitivity, $x_{uv}\wedge x_{vw}\to x_{uw}$.
Each pair $e$ carries an independent answer $y_e=x^\star_e\oplus\xi_e$ with $\xi_e\sim\Bern(\delta)$.
In this section the true partition $x^\star$ has a single cluster, so every flip is a false ``different'', and $\delta$ is the false-split rate.
The loss $L$ is the number of pairs on which the output partition disagrees with the truth; here it is the number of pairs that the output separates.

The pipeline $U_\pi$ is the transitivity pruning of \citet{wang2013leveraging}.
It maintains a partition of $[m]$ into blocks and a set of $\neq$ constraints between blocks.
It processes the pairs in the order $\pi$.
For a pair $e=\{u,v\}$, it skips $e$ if $u$ and $v$ lie in the same block or if their blocks carry a $\neq$ constraint.
Otherwise it queries the answer $y_e$.
On ``same'' it merges the two blocks, and the merged block inherits all $\neq$ constraints of both; on ``different'' it adds a $\neq$ constraint between the two blocks.
Its output is the final partition into blocks.

For an order $\pi$ of the pairs, $\Tpi$ is the spanning tree obtained by scanning $\pi$ and keeping a pair if and only if it joins two different components of the pairs kept so far.
This is Kruskal's algorithm \citep{kruskal1956shortest} for the edge weights ``position in $\pi$'', so $\Tpi$ is the minimum spanning tree for these weights.
It depends on $\pi$ only, not on the answers. With all answers correct, its edges are exactly the queried pairs; we therefore use it as a reference tree even in noisy runs.
For a spanning tree $T$ of the complete graph $K_m$, $d_T(u,v)$ is the number of edges on the path from $u$ to $v$ in $T$.
For an edge $e\in T$, $n_e$ and $m-n_e$ are the numbers of vertices of the two components of $T-e$.
The \emph{Wiener index} of $T$ is
\[
\Wien(T)=\sum_{\{u,v\}}d_T(u,v)=\sum_{e\in T}n_e(m-n_e),
\]
where the second form counts, for each edge $e$, the $n_e(m-n_e)$ pairs whose path uses $e$.
We also write $\Wien_2(T)=\sum_{\{u,v\}}d_T(u,v)^2$.

\begin{lemma}[Union-find is the greedy decoder]\label{lem:ufgreedy}
For every order $\pi$ of the pairs and all answers, the union-find pipeline $U_\pi$ equals the greedy decoder on the partition code.
That is, it skips exactly the pairs whose relation is entailed by the answers accepted so far, it accepts the answer on every other pair, and its output is the unique partition that satisfies all accepted answers.
\end{lemma}

\begin{proof}
Let $L$ be a set of accepted literals, each an equality $u\sim v$ or an inequality $u\not\sim v$, and assume that $L$ is satisfiable, that is, that some partition satisfies all of them.
Call the connected components of the graph of accepted equalities the \emph{equality-components}.
We claim: (a) $L$ entails $u\sim v$ if and only if $u$ and $v$ lie in the same equality-component; (b) $L$ entails $u\not\sim v$ if and only if some accepted inequality joins the equality-component of $u$ to that of $v$.

(a) If $u$ and $v$ are joined by a path of accepted equalities, transitivity gives $u\sim v$ in every model.
Conversely, consider the partition into equality-components.
It satisfies every accepted equality.
It also satisfies every accepted inequality: in any model every equality-component lies in one block, so a satisfiable $L$ contains no inequality inside an equality-component.
This partition separates any two points of different equality-components, so for such points $L$ does not entail $u\sim v$.

(b) If an accepted inequality joins a point $x$ of the component of $u$ to a point $y$ of the component of $v$, then every model has $u\sim x$, $x\not\sim y$ and $y\sim v$, hence $u\not\sim v$ by transitivity.
Conversely, suppose that no accepted inequality joins the two components.
If the two components coincide, $L$ entails $u\sim v$ and, being satisfiable, does not entail $u\not\sim v$.
Otherwise, merge these two components and keep all other equality-components as blocks.
The resulting partition satisfies every accepted equality and every accepted inequality, and in it $u\sim v$; so $L$ does not entail $u\not\sim v$.

We now compare the two decoders step by step.
Suppose that before some step both have accepted the same satisfiable set $L$; this holds at the start, with $L=\emptyset$.
The blocks of the pipeline are the equality-components of $L$, because a merge happens exactly when an equality is accepted.
Two blocks carry a $\neq$ constraint if and only if an accepted inequality joins them, because constraints are created by accepted inequalities and inherited through merges.
By (a) and (b), the pipeline skips the current pair exactly when $L$ entails its relation, which is the rule of the greedy decoder; otherwise both decoders accept the answer.
An accepted answer is consistent with $L$, because its negation is not entailed, so the new set is again satisfiable.
By induction, the two decoders accept the same literals at every step.
At the end, the final blocks form the partition into equality-components, which satisfies $L_n$ by the proof of (a); by \cref{lem:welldef}(ii) it is the unique model of $L_n$, that is, the output of the greedy decoder.
\end{proof}

The full flip vector $\xi$ assigns a potential error to every pair, but the pipeline reads $\xi_e$ only when it queries that pair.
The decision to query $e$ uses only earlier answers, so its unobserved error is still independent of that decision. This gives the same distribution as independently flipping each queried answer.

\begin{proof}[Proof of \cref{lem:sandwich}]
Let the truth be one cluster, and fix an order $\pi$ and a flip vector $\xi$.
Let $F=\{e\in\Tpi:\ \xi_e=1\}$ be the set of flipped tree edges.
For a pair $\{u,v\}$, let $k(u,v)$ be the number of edges of $F$ on the path from $u$ to $v$ in $\Tpi$.
We show that $k(u,v)=0$ implies that $u$ and $v$ end in the same block, and that $k(u,v)=1$ implies that they end in different blocks.
Consequently $N_1\le L_\pi\le N_{\ge1}$, where $N_j$ (respectively $N_{\ge j}$) is the number of pairs with $k(u,v)=j$ (respectively $k(u,v)\ge j$).

Let $\kappa_t$ be the partition of $[m]$ into the components of the edges of $\Tpi$ among the first $t$ pairs of $\pi$.
It equals the partition into the components of all of the first $t$ pairs, because Kruskal's algorithm keeps a pair exactly when it joins two components.
Let $P_t$ be the partition of the pipeline into blocks after $t$ pairs.

(A) $P_t$ refines $\kappa_t$: every pipeline block lies inside one $\kappa_t$-component.
We use induction on $t$; blocks change only by merges.
A merge at step $t$ happens on a queried pair $e=\{u,v\}$.
After step $t$, $u$ and $v$ lie in the same $\kappa_t$-component: either $e$ joins two $\kappa_{t-1}$-components, and then $e\in\Tpi$, or $u$ and $v$ were already in one component.
The merged block is the union of two blocks, each contained in a $\kappa_{t-1}$-component by the induction hypothesis, and both of these components lie inside the $\kappa_t$-component of $u$.

(B) At every time $t$, every $\neq$ constraint joins two blocks that lie in the same $\kappa_t$-component.
A constraint is created at some step on a queried pair $\{u,v\}$, and from then on it joins the current block of $u$ and the current block of $v$, because constraints are inherited through merges.
As in (A), $u$ and $v$ lie in one $\kappa$-component from that step on, since the partitions $\kappa_t$ only become coarser.
By (A), the current blocks of $u$ and $v$ lie inside this component.

(C) Every edge $e=\{u,v\}\in\Tpi$ is queried.
At its step $t$, the points $u$ and $v$ lie in different $\kappa_{t-1}$-components, by the definition of $\Tpi$.
By (A) they lie in different blocks.
By (B) no $\neq$ constraint joins these two blocks.
So $e$ is queried.
Since the truth is one cluster, the query merges the two blocks if $\xi_e=0$, and it creates a $\neq$ constraint between them if $\xi_e=1$.

Blocks never split.
By (C), the two ends of every edge of $\Tpi\setminus F$ end in the same block.
So every component of $\Tpi\setminus F$ (a \emph{fragment}) lies inside one final block.
This is the first claim, because $k(u,v)=0$ means that $u$ and $v$ lie in the same fragment.
For the second claim, let $f=\{x,y\}$ be the unique flipped edge on the path from $u$ to $v$, with $u$ on the side of $x$.
By (C), at the step of $f$ a $\neq$ constraint is placed between the block of $x$ and the block of $y$.
Two blocks joined by a $\neq$ constraint are never merged, because a merge follows only a query, and a query is made only between unconstrained blocks.
The constraint is inherited by every later block that contains $x$ or $y$.
So $x$ and $y$ end in different blocks.
The paths from $u$ to $x$ and from $y$ to $v$ contain no flipped edge, so by the first claim $u$ ends with $x$ and $v$ ends with $y$.
Hence $u$ and $v$ end in different blocks.
Since the loss counts the separated pairs, the pairs with $k=1$ are counted and the pairs with $k=0$ are not, which gives $N_1\le L_\pi\le N_{\ge1}$.
\end{proof}

\begin{proof}[Proof of \cref{thm:wiener}]
We show that for every order $\pi$ and every $\delta\in[0,1]$, with $d=d_{\Tpi}(u,v)$ in each summand,
\[
\delta\,\Wien(\Tpi)-\delta^2\Wien_2(\Tpi)\ \le\ \sum_{\{u,v\}}d\,\delta(1-\delta)^{d-1}\ \le\ \E_\xi L_\pi\ \le\ \sum_{\{u,v\}}\big(1-(1-\delta)^{d}\big)\ \le\ \delta\,\Wien(\Tpi).
\]
In particular $\E_\xi L_\pi=\delta\,\Wien(\Tpi)+O(\delta^2)$, with remainder at most $\delta^2\Wien_2(\Tpi)\le\delta^2\binom m2(m-1)^2$.
Moreover, a single flip on a tree edge $e$ costs exactly $n_e(m-n_e)$, and a single flip on a pair outside the tree costs nothing.

Take expectations over $\xi$ in \cref{lem:sandwich}.
The tree $\Tpi$ does not depend on $\xi$, so the flips on the $d$ edges of the path from $u$ to $v$ are i.i.d.\ $\Bern(\delta)$, and $k(u,v)$ is binomial with parameters $d$ and $\delta$.
Hence $\Prob(k(u,v)=1)=d\delta(1-\delta)^{d-1}$ and $\Prob(k(u,v)\ge1)=1-(1-\delta)^d$, which gives the two middle inequalities.
For the outer ones, $d\delta(1-\delta)^{d-1}\ge d\delta(1-(d-1)\delta)\ge d\delta-d^2\delta^2$ by Bernoulli's inequality, and $1-(1-\delta)^d\le d\delta$.
Summing over pairs gives the bounds, and $\Wien_2(\Tpi)\le\binom m2(m-1)^2$ because $d\le m-1$.
With exactly one flip, on the pair $e$, we have $k(u,v)\in\{0,1\}$ for all pairs.
If $e\in\Tpi$, then $L_\pi=N_1$ is the number of pairs whose tree path uses $e$, that is, $n_e(m-n_e)$.
If $e\notin\Tpi$, then $F=\emptyset$ and $L_\pi=0$; indeed, the pair $e$ is never queried when all earlier answers are correct.
\end{proof}

With no wrong answer, the pipeline makes exactly $m-1$ queries, namely the edges of $\Tpi$.
By (C) in the proof of \cref{lem:sandwich} every tree edge is queried.
A pair outside the tree arrives when its two ends already lie in one $\kappa$-component; with all answers correct, every tree edge so far has merged its blocks, so the blocks coincide with the $\kappa$-components and the pair is skipped.

\paragraph{Known results on the random minimum spanning tree.}
For uniform $\pi$, the tree $\Tpi$ is distributed as the minimum spanning tree $\MST_m$ of $K_m$ with i.i.d.\ continuous edge weights.
We use two known results on $\MST_m$.
\begin{itemize}
\item[(F1)] \citep{addarioberry2009critical} $\E\operatorname{diam}(\MST_m)=\Theta(m^{1/3})$, where $\operatorname{diam}$ is the largest tree distance between two vertices.
\item[(F2)] \citep[Theorem~1.1]{addarioberry2017scaling} The tree $\MST_m$, with distances multiplied by $m^{-1/3}$ and with the uniform probability measure on its vertices, converges in distribution, in the Gromov--Hausdorff--Prokhorov topology, to a random compact measured real tree $(\mathcal M,d,\mu)$. Let $D_m$ be the tree distance between two independent uniform vertices of $\MST_m$. By the standard Gromov--Prokhorov consequence of this convergence, $D_m/m^{1/3}$ converges in distribution to $D:=d(\zeta_1,\zeta_2)$, where, given $(\mathcal M,d,\mu)$, the points $\zeta_1$ and $\zeta_2$ are independent with law $\mu$. \citet[Theorem~1.1 and Proposition~6.5]{broutin2023convex} identify the same measured limit and show that, for a point $\zeta$ sampled from $\mu$, almost surely $\mu(B(\zeta,r))\le r^{5/2}$ for all sufficiently small $r>0$. Thus $\mu(\{\zeta\})=0$ almost surely. Conditional independence then gives $\Prob(D=0)=\E[\mu(\{\zeta_1\})]=0$, so $\Prob(D>0)=1$.
\end{itemize}

The scale $m^{-1/3}$ also appears in this tree's edge-weight-resampling noise threshold \citep{israeli2023noise}.

\begin{proof}[Proof of \cref{prop:random}]
First, $\E\,\Wien(\MST_m)=\Theta(m^{7/3})$, so that $\E L=\delta\,\E\,\Wien(\MST_m)+O(\delta^2)=\Theta(\delta m^{7/3})+O(\delta^2)$, an amplification $\E\,\Wien(\MST_m)/\binom m2=\Theta(m^{1/3})$ over the raw rate.
Second, let $\delta=x\,m^{-1/3}$ with $x>0$ fixed, and let $\varphi_m(x)=\E L/\binom m2$ be the expected fraction of separated pairs. Then
\[
\E\big[xD\,e^{-xD}\big]\ \le\ \liminf_{m\to\infty}\varphi_m(x)\ \le\ \limsup_{m\to\infty}\varphi_m(x)\ \le\ 1-\E\,e^{-xD}\ \le\ x\,\E D,
\qquad 0<\E D<\infty.
\]

\emph{The $m^{7/3}$ law.}
For the upper bound, $\Wien(T)\le\binom m2\operatorname{diam}(T)$ for every tree $T$, and (F1) gives $\E\,\Wien(\MST_m)=O(m^{7/3})$.
For the lower bound, let $D'_m$ be the tree distance between a uniformly random pair of distinct vertices.
Then $\E\,\Wien(\MST_m)=\binom m2\E D'_m$, and $\E D'_m=\E[D_m\mid\text{the two vertices differ}]\ge\E D_m$ because $D_m=0$ when the two vertices coincide.
By (F2), $D_m/m^{1/3}$ converges in distribution to $D$.
For nonnegative random variables, convergence in distribution implies $\liminf\E[D_m/m^{1/3}]\ge\E D$ (Fatou's lemma, or the portmanteau theorem applied to the nonnegative continuous map $t\mapsto t$).
Since $\Prob(D>0)>0$, $\E D>0$, so $\E\,\Wien(\MST_m)=\Omega(m^{7/3})$.
The first-order statement then follows from \cref{thm:wiener} by averaging over $\pi$: $\delta\,\E\,\Wien(\MST_m)-\delta^2\,\E\,\Wien_2(\MST_m)\le\E L\le\delta\,\E\,\Wien(\MST_m)$.

\emph{The window.}
Here the error rate shrinks on the reciprocal scale of typical tree distances, so a typical path encounters a constant-order number of errors.
By \cref{lem:sandwich} and the independence of the flips from $\Tpi$, as in the proof of \cref{thm:wiener}, $\varphi_m(x)$ lies between $\E\big[D'_m\delta(1-\delta)^{D'_m-1}\big]$ and $\E\big[1-(1-\delta)^{D'_m}\big]$.
The laws of $D'_m$ and $D_m$ differ only by an event of probability $1/m$ (the two uniform vertices coincide), on which $D_m=0$.
So $D'_m/m^{1/3}$ also converges in distribution to $D$.
Write $(1-\delta)^{D'_m}=\exp\big(c_m\,D'_m/m^{1/3}\big)$ with $c_m=m^{1/3}\ln(1-x m^{-1/3})\to-x$.
By the continuous mapping theorem, $(1-\delta)^{D'_m}$ converges in distribution to $e^{-xD}$, and $D'_m\delta(1-\delta)^{D'_m-1}$ converges in distribution to $xDe^{-xD}$.
Both variables are bounded: the first by~1, and the second by $\sup_{y\ge0}ye^{-y}/(1-\delta)$, because $d\delta(1-\delta)^{d-1}\le y e^{-y}/(1-\delta)$ with $y=-d\ln(1-\delta)\ge d\delta$.
So their expectations converge to $\E[xDe^{-xD}]$ and $1-\E e^{-xD}$, which gives the two middle inequalities.
The last inequality follows from $1-e^{-t}\le t$.
Finally, $D_m\le\operatorname{diam}(\MST_m)$, so Fatou's lemma and (F1) give $\E D\le\liminf_m\E D_m/m^{1/3}\le\liminf_m\E\operatorname{diam}(\MST_m)/m^{1/3}<\infty$; and $\E D>0$ as shown above.

\emph{Interpretation.}
As $x\to0$ both bounds equal $x\,\E D\,(1+o(1))$; for the lower bound this follows by monotone convergence of $De^{-xD}\uparrow D$, and for the upper bound by monotone convergence of $(1-e^{-xD})/x\uparrow D$.
Thus $\liminf_m\varphi_m(x)$ and $\limsup_m\varphi_m(x)$ both behave as $x\,\E D$ as $x\to0$.
Inside the window, for each fixed $x$, the loss fraction is of constant order, bounded away from~0 (because $\E D>0$) and from~1 (because $D<\infty$ almost surely).
\end{proof}

Whether $\varphi_m(x)$ has a limit $\Lambda(x)$, and whether $\Lambda(x)\to1$ or stays below~1 as $x\to\infty$, depends on how the pipeline repairs pairs with at least two flipped edges on their path; this question is open.
For a fixed query order, the loss need not increase when additional answers are flipped.
Suppose that removing two tree edges $f_1,f_2$ from $\Tpi$ leaves three fragments $P_1,P_2,P_3$, with $f_1$ between $P_1$ and $P_2$ and $f_2$ between $P_2$ and $P_3$.
If only $f_1$ is flipped, \cref{lem:sandwich} gives the final partition $\{P_1,\,P_2\cup P_3\}$, and similarly for $f_2$.
If both are flipped, the run can end in the partition $\{P_1\cup P_3,\,P_2\}$, for instance when a pair between $P_1$ and $P_3$ arrives after both flipped edges; it is then queried and answered ``same''.
When $P_2$ has fewer points than each of $P_1$ and $P_3$, this separates fewer pairs than flipping either edge alone.
So, for a fixed order, flipping more answers can decrease the loss.

\begin{proof}[Proof of \cref{prop:star}]
(i) Every spanning tree $T$ of $K_m$ with $m\ge3$ has $\Wien(T)\ge(m-1)^2$, with equality if and only if $T$ is a star; hence, among all orders, the first-order risk $\delta\,\Wien(\Tpi)$ of $U_\pi$ is minimized exactly by the orders whose Kruskal tree is a star (for example the pivot orders).
(ii) Let a decoder observe only the answers on the $m-1$ pairs of a fixed spanning tree $T$ and output any partition, possibly at random. Let $\theta_0$ be the one-cluster partition, and for $e\in T$ let $\theta_e$ be the partition into the two components of $T-e$. If $\delta\le1/m$, then
\[
\max_{\theta\in\{\theta_0\}\cup\{\theta_e\}_{e\in T}}\E_\theta L\ \ge\ \tfrac12\,\delta(1-\delta)^{m-2}\,\Wien(T)\ \ge\ \tfrac12\,\delta(1-\delta)^{m-2}(m-1)^2,
\]
where $\E_\theta L$ is the expected number of pairs on which the output disagrees with $\theta$ when the truth is $\theta$.

(i)
For $1\le n_e\le m-1$ we have $n_e(m-n_e)\ge m-1$, with equality if and only if $n_e\in\{1,m-1\}$.
So $\Wien(T)=\sum_{e\in T}n_e(m-n_e)\ge(m-1)^2$.
Equality holds if and only if every edge has a leaf as an endpoint, which for a tree on $m\ge3$ vertices means that $T$ is a star.
An order that lists the $m-1$ pairs $\{1,j\}$ first, as in the node-ordering (pivot) pipeline of \citet{vesdapunt2014crowdsourcing}, has the star centered at~1 as its Kruskal tree.
Its first-order risk is therefore exactly $\delta(m-1)^2$, an amplification $(m-1)^2/\binom m2=2(m-1)/m<2$ over the raw rate.
With no wrong answer, both the pivot order and a random order make exactly $m-1$ queries (the remark after the proof of \cref{thm:wiener}).

(ii)
Let $z_e$ be the noiseless observation on $T$ under $\theta_e$, namely ``different'' on $e$ and ``same'' on the other edges of $T$; these observations are distinct for distinct $e$.
Let $\hat\theta$ be the decoder, and let $\ell(\cdot,\cdot)$ denote the number of pairs on which two partitions disagree.
Put prior weight $\frac12$ on $\theta_0$ and weight $\frac1{2(m-1)}$ on each $\theta_e$.
Keeping only the observations $z_e$, the Bayes risk of $\hat\theta$ is at least
\[
\sum_{e\in T}\Big[\tfrac12\,P_{\theta_0}(z_e)\,\E\,\ell(\hat\theta(z_e),\theta_0)+\tfrac1{2(m-1)}\,P_{\theta_e}(z_e)\,\E\,\ell(\hat\theta(z_e),\theta_e)\Big],
\]
where the expectations are over the internal randomness of the decoder.
Here $P_{\theta_0}(z_e)=\delta(1-\delta)^{m-2}$ and $P_{\theta_e}(z_e)=(1-\delta)^{m-1}$.
For $\delta\le1/m$ we have $1-\delta\ge(m-1)\delta$, hence $\frac{(1-\delta)^{m-1}}{m-1}\ge\delta(1-\delta)^{m-2}$.
The distance $\ell$ is a Hamming distance on pair indicators, so $\ell(\hat\theta,\theta_0)+\ell(\hat\theta,\theta_e)\ge\ell(\theta_0,\theta_e)=n_e(m-n_e)$.
So the Bayes risk is at least $\frac12\delta(1-\delta)^{m-2}\sum_{e\in T}n_e(m-n_e)=\frac12\delta(1-\delta)^{m-2}\Wien(T)$.
The maximum risk is at least the Bayes risk, and (i) gives the second inequality.
Part~(ii) assumes that the observed tree is fixed in advance. The pipeline $U_\pi$ can change its queries in response to answers, even though its noiseless query set is $\Tpi$, so the bound does not directly apply to it.
\end{proof}

\paragraph{Similarity order.}
The items carry embeddings $X_1,\dots,X_m$, i.i.d.\ with a density $f$ on $[0,1]^d$ that is bounded below by $f_0>0$.
The pairs are processed in increasing order of the Euclidean distance $|X_u-X_v|$ (most similar first), as in the order by decreasing match probability of \citet{wang2013leveraging}.
Then $\Tpi$ is the Euclidean minimum spanning tree $\EMST_m$ of the points, which is almost surely unique.
The order is independent of the answers, so \cref{thm:wiener} gives the first-order risk $\delta\,\Wien(\EMST_m)$ conditionally on the embeddings.
Throughout, $\log$ is the natural logarithm.

\begin{proof}[Proof of \cref{thm:similarity}]
Let $\mu_f=\E|X_1-X_2|>0$.
We show that for every $m\ge2$ with $f_0m\ge2\log m$,
\[
\frac{\E\,\Wien(\EMST_m)}{\binom m2}\ \ge\ \frac{\mu_f-\sqrt d/m}{4\sqrt d}\Big(\frac{f_0\,m}{2\log m}\Big)^{1/d}=c_{d,f}\Big(\frac m{\log m}\Big)^{1/d}\big(1-o(1)\big),
\qquad c_{d,f}=\frac{\mu_f}{4\sqrt d}\Big(\frac{f_0}2\Big)^{1/d},
\]
and that for $d=1$ and any continuous distribution $\Wien(\EMST_m)/\binom m2=(m+1)/3$ exactly.

The proof bounds the length of the longest tree edge, then uses it to show that paths between distant points must contain many edges.

(1) \emph{Bottleneck property.}
For $r>0$ let $G(X;r)$ be the geometric graph on the points with an edge between any two points at distance at most $r$.
The longest edge $\ell_{\max}$ of $\EMST_m$ is at most any $r$ for which $G(X;r)$ is connected.
Indeed, suppose that $G(X;r)$ is connected and that an edge $e$ of $\EMST_m$ is longer than $r$.
Removing $e$ splits the points into two sides.
Some edge of $G(X;r)$ crosses this cut, and exchanging it for $e$ gives a lighter spanning tree, which is a contradiction.

(2) \emph{Tiling.}
Let $x=(f_0m/(2\log m))^{1/d}$, which is at least~1 by assumption, and let $N=\lfloor x\rfloor\ge x/2$.
Tile $[0,1]^d$ by $N^d$ cubes of side $1/N$.
Each cube receives a given point with probability at least $f_0N^{-d}\ge f_0x^{-d}=2\log m/m$.
So the probability that some cube is empty is at most $N^d(1-2\log m/m)^m\le x^dm^{-2}=f_0/(2m\log m)\le1/m$; the last step uses $f_0\le1\le2\log m$, where $f_0\le1$ because $f$ integrates to~1 on the unit cube.
Suppose that no cube is empty.
Two points in the same cube are within distance $\sqrt d/N$, and two points in face-adjacent cubes are within distance $\sqrt{d+3}/N\le2\sqrt d/N$.
The face-adjacency graph of the cubes is connected, and every point lies in a cube.
So $G(X;2\sqrt d/N)$ is connected, and by (1), $\ell_{\max}\le2\sqrt d/N\le4\sqrt d/x$.
Call this event $A$; we have shown $\Prob(A^c)\le1/m$.

(3) \emph{Paths are long.}
On $A$, the tree path from $u$ to $v$ consists of $d_T(u,v)$ edges of length at most $\ell_{\max}$, and it covers the Euclidean distance $|X_u-X_v|$.
So $d_T(u,v)\ge|X_u-X_v|/\ell_{\max}\ge x|X_u-X_v|/(4\sqrt d)$, where $T=\EMST_m$.
Hence
\[
\E\,\Wien(\EMST_m)\ \ge\ \frac{x}{4\sqrt d}\,\E\Big[\ind_A\sum_{\{u,v\}}|X_u-X_v|\Big]\ \ge\ \frac{x}{4\sqrt d}\binom m2\Big(\mu_f-\frac{\sqrt d}m\Big),
\]
because $|X_u-X_v|\le\sqrt d$ and $\Prob(A^c)\le1/m$.
Dividing by $\binom m2$ gives the bound.

(4) \emph{The case $d=1$.}
For distinct points on a line, which occur almost surely under a continuous distribution, the minimum spanning tree is the path through the points in sorted order.
The path $P_m$ has $\Wien(P_m)=\sum_{i=1}^{m-1}i(m-i)=\binom{m+1}3$, so $\Wien(P_m)/\binom m2=(m+1)/3$.

Relative to \cref{prop:random}, the similarity-first amplification is larger by order $m^{2/3}$ in $d=1$ and at least order $m^{1/6}/\sqrt{\log m}$ in $d=2$.
Thus, under these embedding assumptions and for all sufficiently large $m$, similarity ordering within a cluster has higher first-order risk than random ordering and pivot ordering.
\end{proof}

For $d=2$, sharper amplification exponents would require finer control of planar MST path geometry; existing planar scaling results give bounds on path dimension \citep{aizenman1999scaling}.
For $d\ge3$ the bound $(m/\log m)^{1/d}$ is below $m^{1/3}$ and gives no comparison with random order.

\paragraph{Several clusters.}\label{app:proofs-multi}
There are $K$ latent clusters $C_1,\dots,C_K$ of sizes $s_1,\dots,s_K$, and $n=\sum_is_i$ items.
All $\binom n2$ pairs are processed in a uniformly random order $\pi$ by the union-find pipeline $U_\pi$ of \cref{app:proofs-er}.
A pair inside a cluster is a \emph{within} pair, and a pair between two clusters is a \emph{cross} pair.
The answer on a within pair is flipped to a false ``different'' with probability $\eps_w$, and the answer on a cross pair is flipped to a false ``same'' with probability $\eps_x$; all flips are independent.
The loss is the number of false-positive (FP) pairs, which the output wrongly puts together, plus the number of false-negative (FN) pairs, which the output wrongly separates.
We generate the uniform order by attaching independent $\mathrm{Unif}[0,1]$ arrival times to the pairs.
For $N\ge1$ and $p\in[0,1]$, $G(N,p)$ is the Erd\H{o}s--R\'enyi random graph on $N$ vertices, in which each pair of vertices is an edge independently with probability $p$.
As in \cref{app:proofs-er}, $\MST_s$ is the minimum spanning tree of $K_s$ with i.i.d.\ continuous edge weights, and $\Wien$ is the Wiener index.

For sizes $s_A$ and $s_B$, consider only two clusters $A$ and $B$, with all pairs of $A\cup B$ in uniformly random order.
All answers are correct except one cross pair $e=\{u,v\}$, with $u\in A$ and $v\in B$, which is answered ``same''.
Let $D_e$ be the loss, and let $g(s_A,s_B)=\E_\pi D_e$; by symmetry it does not depend on the choice of $e$.
We split $g=g_{\mathrm{FP}}+g_{\mathrm{FN}}$ into the expected numbers of FP and FN pairs, and we write $g(s)=g(s,s)$.

\begin{proof}[Proof of \cref{thm:multicluster}]
We show that, with $\eps=\max(\eps_w,\eps_x)$ and $N=\binom n2$,
\[
R_U(\eps_w,\eps_x)=\eps_w\sum_{i=1}^K\E\,\Wien(\MST_{s_i})+\eps_x\sum_{i<j}s_is_j\,g(s_i,s_j)+O(\eps^2),
\]
with a remainder of absolute value at most $\tfrac32\eps^2N^2(N-1)$.
For equal sizes $s$ this reads $R_U=\eps_wK\,\E\,\Wien(\MST_s)+\eps_x\binom K2s^2g(s)+O(\eps^2)$.

The pipeline is exact on the noiseless input, and its output is a function of the answers and of the order, which is independent of the answers.
By \cref{lem:reduction} with per-coordinate rates (end of its proof), it suffices to compute the single-error damages. We treat within-cluster and cross-cluster errors separately, since they create false splits and false merges by different mechanisms.

\emph{Within flip on a pair $e\subseteq C_i$.}
All cross answers are correct (``different''), so no block ever contains items of two clusters, since a merge needs a ``same'' answer.
A $\neq$ constraint between two blocks of the same cluster is created only by a query between them answered ``different'', that is, only by $e$.
Constraints between blocks of different clusters never affect a decision on a within pair.
Hence cluster $C_i$ evolves exactly as the single-cluster pipeline under the order $\pi$ restricted to its pairs, which is uniform, with one flip; and the other clusters end intact.
By \cref{thm:wiener} the damage is $n_e(s_i-n_e)\,\ind\{e\in T_{\pi|C_i}\}$, where $T_{\pi|C_i}$ is the Kruskal tree of the restricted order.
The sum of these damages over the within pairs of $C_i$ is $\Wien(T_{\pi|C_i})$, whose expectation is $\E\,\Wien(\MST_{s_i})$.

\emph{Cross flip on a pair $e$ between $C_i$ and $C_j$.}
Let $S=C_i\cup C_j$.
We use induction over time.
(a) Every block is contained in $S$ or disjoint from $S$, and every block disjoint from $S$ lies inside one cluster.
This holds because a merge of blocks from different clusters needs a ``same'' answer on a cross pair, and only $e\subseteq S$ has one.
(b) A $\neq$ constraint between two blocks inside $S$ is created only by a query of a pair inside $S$, and it is inherited only through merges inside $S$.
Pairs with an endpoint outside $S$ only create constraints that involve a block outside $S$, or merge two blocks of one cluster outside $S$; neither changes a decision on a pair inside $S$.
Therefore the state restricted to $S$ (its blocks and the constraints among them) changes only at the arrivals of pairs inside $S$, as a deterministic function of their relative order, which is uniform.
The clusters outside $S$ end intact, as in the within case without a flip, and they are separated from everything else.
So the expected loss is exactly $g(s_i,s_j)$.
Summing over the $s_is_j$ cross pairs between $C_i$ and $C_j$ gives the formula.
\end{proof}

\begin{lemma}[Anatomy of one false merge]\label{lem:anatomy}
In the two-cluster setting of $g$, let $\tau$ be the arrival time of $e$.
Let $P$ be the component of $u$ in the graph formed by the pairs inside $A$ that arrive before $\tau$, and let $Q$ be the component of $v$ in the graph formed by the pairs inside $B$ that arrive before $\tau$.
Let $q$ denote the event that $e$ is queried.
\begin{enumerate}
\item[1.] (Query event.) The pair $e$ is queried if and only if no pair of $P\times Q$ other than $e$ arrives before $\tau$. Hence $\Prob(q\mid\tau,P,Q)=(1-\tau)^{|P||Q|-1}$. If $e$ is not queried, then $D_e=0$.
\item[2.] (Final form.) If $e$ is queried, the final partition of $A\cup B$ is $\{X,\ A\setminus X_A,\ B\setminus X_B\}$, with empty sets dropped. Here $X$ is the final block of $u$, which contains $u$ and $v$, and $X_A=X\cap A$, $X_B=X\cap B$. With $a=|X_A|$ and $b=|X_B|$,
\[
D_e=\underbrace{ab}_{\mathrm{FP}}+\underbrace{a(s_A-a)+b(s_B-b)}_{\mathrm{FN}}.
\]
\item[3.] (Mechanism.) After $\tau$, when a pair arrives between the current block $R\neq X$ inside $A$ and the block $X$, and the two blocks carry no constraint, the pair is queried. If it lies inside $A$, it is answered ``same'' and $R$ joins $X$. If it is a cross pair between $R$ and $X_B$, it is answered ``different'' (correctly), and $R$, together with every block that later merges with $R$, stays outside $X$, because the $\neq$ constraint is inherited. The same holds with $A$ and $B$ exchanged. A $\neq$ constraint inherited from before $\tau$, created by a cross pair between $R$ and $Q$ that arrived before $\tau$, excludes $R$ from the start.
\item[4.] (Bounds on the damage.) On the query event, $\max(s_A,s_B)-1\le D_e\le\binom{s_A+s_B}2$. For $s_A=s_B=s$, $s-1\le D_e\le s^2$.
\end{enumerate}
\end{lemma}

\begin{proof}
1. Before $\tau$ all answers are correct.
Two blocks of the same cluster never carry a constraint, because all within answers are ``same''.
So the blocks are the components of the within-cluster graphs of the pairs that have arrived.
By induction over time, a block $P'$ inside $A$ and a block $Q'$ inside $B$ carry a constraint if and only if some pair of $P'\times Q'$ has arrived.
Indeed, when such a pair arrives, the blocks of its two endpoints (which lie inside $P'$ and $Q'$) either already carry a constraint, or the pair is queried and answered ``different'', which creates one; and constraints persist under merges.
Conversely, a constraint is created only by a query on an arrived pair between the two blocks involved.
At time $\tau$ the blocks of $u$ and $v$ are $P$ and $Q$, and $e$ is queried if and only if they carry no constraint.
Given $\tau$ and the within-cluster graphs, the $|P||Q|-1$ cross pairs of $P\times Q$ other than $e$ arrive before $\tau$ independently, each with probability $\tau$.
If $e$ is skipped, every accepted answer is correct, and the output is the truth.

2. After $\tau$, the only block that contains items of both clusters is $X$, because no other cross answer is ``same''.
Two blocks inside $A$ and outside $X$ never carry a constraint against each other: such a constraint needs a query between them answered ``different'', but all answers on pairs inside $A$ are ``same''.
So all items of $A$ outside $X$ end in one block, since the first arriving pair between two such blocks merges them; the same holds for $B$; and $X$ never splits.
The loss formula counts the pairs of $X_A\times X_B$ (wrongly together) and the pairs of $X_A\times(A\setminus X_A)$ and $X_B\times(B\setminus X_B)$ (wrongly apart).

3. This is the rule of the pipeline applied to the blocks $R$ and $X$, together with the inheritance of constraints.

4. We have $a,b\ge1$.
If $a<s_A$, then $a(s_A-a)\ge s_A-1$; if $a=s_A$, then $ab\ge s_A$.
In both cases $D_e\ge s_A-1$, and the same argument gives $D_e\ge s_B-1$.
The upper bound $\binom{s_A+s_B}2$ is the number of pairs.
For equal sizes, $ab+a(s-a)+b(s-b)=s(a+b)-a^2-b^2+ab\le s^2$, since on $[0,s]^2$ this concave function is maximized at $a=b=s$.
\end{proof}

\begin{lemma}[Disconnection tail]\label{lem:disconnect}
For two fixed distinct vertices $u$ and $w$ of $G(N,p)$, let $h_N(p)$ be the probability that $u$ and $w$ lie in different components.
Then $h_N(p)\le e^5e^{-Np/2}$ for all $N\ge2$ and $p\in[0,1]$.
\end{lemma}

\begin{proof}
Two vertices in different components cannot both lie in components with more than $N/2$ vertices.
So $h_N(p)\le2\,\Prob(|C(u)|\le N/2)$, where $C(u)$ is the component of $u$.
If $|C(u)|=k\le N/2$, then some $k$-set $S\ni u$ contains a spanning tree of present edges and has no present edge to its complement; there are $k(N-k)\ge kN/2$ pairs between $S$ and its complement.
For a fixed set $S$ and a fixed tree on $S$, these two events concern disjoint sets of pairs, so they are independent.
There are $\binom{N-1}{k-1}$ choices of $S$ and, by Cayley's formula, $k^{k-2}$ trees on $S$.
A union bound gives
\[
\Prob(|C(u)|=k)\ \le\ \binom{N-1}{k-1}k^{k-2}p^{k-1}(1-p)^{kN/2}\ \le\ e^k(Np)^{k-1}e^{-kNp/2},
\]
where we used $\binom{N-1}{k-1}k^{k-2}\le N^{k-1}k^{k-1}/k!\le N^{k-1}e^k$ and $1-p\le e^{-p}$.
The $k$-th term equals $e\,e^{-Np/2}\,x^{k-1}$ with $x=e\,Np\,e^{-Np/2}$.
Suppose that $Np\ge10$.
The map $y\mapsto ye^{-y/2}$ decreases for $y\ge2$, so $x\le10e\cdot e^{-5}\le0.19$, and $\sum_{k\ge1}x^{k-1}\le1/0.81\le1.3$.
Hence $\Prob(|C(u)|\le N/2)\le1.3\,e\,e^{-Np/2}$ and $h_N(p)\le2.6\,e\,e^{-Np/2}\le e^5e^{-Np/2}$.
If $Np<10$, the right-hand side $e^5e^{-Np/2}$ exceeds~1, and the bound is trivial.
\end{proof}

\begin{proof}[Proof of \cref{prop:gbounds}]
For $s_A=s_B=s\ge2$, we bound the chance that the wrong answer is queried and then combine it with the damage conditional on that event. This gives
\[
\frac1{2s-1}\ \le\ \Prob(q)\ \le\ \frac{12.4}{s+1}+\frac12\Big(\frac9{16}\Big)^{s-1},\qquad
\frac{s-1}{2s-1}\ \le\ g(s)\ \le\ 13\,s,
\]
and that consequently, for equal sizes, the cross-error term $\eps_x\binom K2s^2g(s)$ of \cref{thm:multicluster} lies between $\frac13\eps_x\binom K2s^2$ and $13\,\eps_x\binom K2s^3$.

\emph{Lower bounds.}
If $u$ is isolated in the graph of the pairs inside $A$ that arrived before $\tau$, and $v$ is isolated in the corresponding graph inside $B$, then $P\times Q=\{e\}$, and $e$ is queried by \cref{lem:anatomy}(1).
Given $\tau$, this has probability $(1-\tau)^{2(s-1)}$, and $\int_0^1(1-\tau)^{2(s-1)}d\tau=1/(2s-1)$.
By \cref{lem:anatomy}(1) and~(4), $D_e=0$ off the query event and $D_e\ge s-1$ on it.
So $g(s)\ge(s-1)\Prob(q)\ge(s-1)/(2s-1)$.

\emph{Reduction to one component.}
Given $\tau$, the graphs of the pairs inside $A$ and inside $B$ that arrived before $\tau$ are independent copies of $G(s,\tau)$.
So $P$ and $Q$ are independent, and each is distributed as the component $C_\tau$ of a fixed vertex of $G(s,\tau)$.
Since $(|P|-1)(|Q|-1)\ge0$, we have $|P||Q|-1\ge(|P|-1)+(|Q|-1)$.
With \cref{lem:anatomy}(1) this gives
\[
\Prob(q\mid\tau)=\E\,(1-\tau)^{|P||Q|-1}\ \le\ \Big(\E\,(1-\tau)^{|C_\tau|-1}\Big)^2.
\]

\emph{Extra-vertex identity.}
Add a vertex $w$ to $G(s,\tau)$ and join it to each of the $s$ vertices independently with probability $\tau$; the result is a copy of $G(s+1,\tau)$.
The vertices $u$ and $w$ are connected in it if and only if $w$ has a neighbor in $C_\tau(u)$, the component of $u$ in $G(s,\tau)$.
On a simple path from $u$ to $w$, the vertex $w$ occurs only at the end; its predecessor is connected to $u$ without $w$ and therefore lies in $C_\tau(u)$.
Given $C_\tau(u)$, the vertex $w$ has no neighbor in it with probability $(1-\tau)^{|C_\tau(u)|}$.
Therefore
\begin{equation}\label{eq:extravertex}
\E\,(1-\tau)^{|C_\tau|}=h_{s+1}(\tau),
\qquad\text{and so}\qquad
\E\,(1-\tau)^{|C_\tau|-1}=\frac{h_{s+1}(\tau)}{1-\tau}\quad(\tau<1),
\end{equation}
with $h_{s+1}$ as in \cref{lem:disconnect}.

\emph{Upper bound on $\Prob(q)$.}
For $\tau\le\frac12$, \eqref{eq:extravertex} and \cref{lem:disconnect} give $\E(1-\tau)^{|C_\tau|-1}\le2e^5e^{-(s+1)\tau/2}$.
Hence $\Prob(q\mid\tau)\le\min\{1,\,4e^{10}e^{-(s+1)\tau}\}$.
For $\alpha\ge1$ and $\lambda>0$ we have $\int_0^\infty\min\{1,\alpha e^{-\lambda\tau}\}d\tau=(1+\ln\alpha)/\lambda$.
So $\int_0^{1/2}\Prob(q\mid\tau)\,d\tau\le(1+\ln(4e^{10}))/(s+1)=(11+\ln4)/(s+1)\le12.4/(s+1)$.
For $\tau>\frac12$ we use $h_{s+1}(\tau)\le(1-\tau)(1-\tau^2)^{s-1}$.
This holds because, for $u$ and $w$ to be disconnected, the edge $uw$ must be absent and, for each of the other $s-1$ vertices $x$, the path $u$--$x$--$w$ must be absent; these $s$ events are independent.
So $\E(1-\tau)^{|C_\tau|-1}\le(1-\tau^2)^{s-1}$, and $\Prob(q\mid\tau)\le(1-\tau^2)^{2(s-1)}\le(3/4)^{2(s-1)}=(9/16)^{s-1}$.
Integrating over $(\frac12,1]$ gives at most $\frac12(9/16)^{s-1}$.
Adding the two parts proves the upper bound on $\Prob(q)$.

\emph{Upper bound on $g$.}
By \cref{lem:anatomy}(4), $D_e\le s^2$ on the query event, so $g(s)\le s^2\Prob(q)\le s^2\big[12.4/(s+1)+\frac12(9/16)^{s-1}\big]$.
This is at most $13s$ for all $s\ge2$.
Indeed, $s\big[12.4/(s+1)+\frac12(9/16)^{s-1}\big]$ equals about $8.83$, $9.77$ and $10.28$ at $s=2,3,4$.
For $s\ge4$ it is at most $12.4+0.36<13$, because $s(9/16)^{s-1}$ decreases in $s$ and equals about $0.71$ at $s=4$.

\emph{Consequence.}
For $s\ge2$ we have $(s-1)/(2s-1)\ge\frac13$, and $s^2g(s)\le13s^3$; this gives the bounds on the cross-error term.
\end{proof}

\paragraph{Limit of the query probability.}
For $c>0$ let $\theta_c$ be the largest solution in $[0,1)$ of $\theta=1-e^{-c\theta}$; thus $\theta_c=0$ for $c\le1$ and $\theta_c>0$ for $c>1$.
For $\mu\in(0,1)$, the $\mathrm{Borel}(\mu)$ law is the law of the total progeny of a Galton--Watson tree with Poisson($\mu$) offspring, $\Prob(K=k)=e^{-\mu k}(\mu k)^{k-1}/k!$ for $k\ge1$.
Fix $c>0$ and let $C_s$ be the component of a fixed vertex in $G(s,c/s)$.
The next proposition uses three standard facts about the Erd\H{o}s--R\'enyi graph \citep[Chs.~3--4]{vanderhofstad2017random}.
\begin{itemize}
\item[(E1)] As $s\to\infty$, $|C_s|$ converges in distribution to the total progeny $T_c\in\{1,2,\dots\}\cup\{\infty\}$ of a Galton--Watson tree with Poisson($c$) offspring, and $\Prob(T_c=\infty)=\theta_c$.
\item[(E2)] For $c>1$, every $M_s\to\infty$ with $M_s=o(s)$ and every $\eta>0$, $\Prob\big(|C_s|>M_s\text{ and }\big||C_s|-\theta_cs\big|>\eta s\big)\to0$: a component that is not small is the giant.
\item[(E3)] For $c>1$, the law of $T_c$ conditioned on $T_c<\infty$ is the $\mathrm{Borel}(\mu_c)$ law with the dual parameter $\mu_c=c(1-\theta_c)$.
\end{itemize}

\begin{proposition}[Limit of the query probability]\label{prop:cq}
Assume (E1)--(E3).
In the two-cluster setting with $s_A=s_B=s$,
\[
\lim_{s\to\infty}s\,\Prob(q)=C_q:=\int_0^\infty f(c)\,dc,\qquad
f(c)=\begin{cases}1,& c\le1,\\[2pt](1-\theta_c)^2+2\theta_c(1-\theta_c)\,\E\,e^{-c\theta_cK},\ \ K\sim\mathrm{Borel}(\mu_c),& c>1.\end{cases}
\]
Numerically, $C_q=1.42864$ (\cref{app:numerics}).
\end{proposition}

\begin{proof}
By \cref{lem:anatomy}(1), $\Prob(q)=\int_0^1\Prob(q\mid\tau)\,d\tau$ with $\Prob(q\mid\tau)=\E(1-\tau)^{|P||Q|-1}$, where $P$ and $Q$ are independent copies of the component of a fixed vertex in $G(s,\tau)$.
The substitution $\tau=c/s$ gives $s\,\Prob(q)=\int_0^s\Prob(q\mid\tau=c/s)\,dc$.

\emph{Pointwise limit.}
Fix $c$. We separate the cases in which the two endpoint components remain finite, one grows to the giant component, or both do; these cases determine whether an earlier cross pair blocks the erroneous query.
Fix $M$ and split on $\{|P|\le M\}$ and $\{|P|>M\}$, and likewise for $Q$; we let $s\to\infty$ first and then $M\to\infty$. By (E1), $\Prob(|P|\le M)\to\Prob(T_c\le M)$, which tends to $1-\theta_c$ as $M\to\infty$; by (E2), on $\{|P|>M\}$ the component is the giant, of size $(\theta_c+o_P(1))s$, up to an event whose probability vanishes in this double limit. The same holds for $Q$, independently of $P$.
If both are finite in the limit (always the case when $c\le1$), then $|P||Q|$ is bounded in probability and $(1-c/s)^{|P||Q|-1}\to1$.
If exactly one is the giant and the other has limit size $K$, then $(1-c/s)^{|P||Q|-1}\to e^{-c\theta_cK}$, and by (E3) $K$ has the $\mathrm{Borel}(\mu_c)$ law.
If both are giants, then $|P||Q|$ is of order $s^2$ and $(1-c/s)^{|P||Q|-1}\to0$.
The three cases have limit probabilities $(1-\theta_c)^2$, $2\theta_c(1-\theta_c)$ and $\theta_c^2$.
The variables $(1-c/s)^{|P||Q|-1}$ lie in $[0,1]$, so their expectations converge, and $\Prob(q\mid\tau=c/s)\to f(c)$.

\emph{Domination.}
For $c\le s/2$, that is, for $\tau\le\frac12$, the proof of \cref{prop:gbounds} gives
\[
\Prob(q\mid\tau=c/s)\ \le\ \min\{1,\,4e^{10}e^{-(s+1)c/s}\}\ \le\ \min\{1,\,4e^{10}e^{-c}\},
\]
and the right-hand side is integrable on $[0,\infty)$.
The same proof also gives $\int_{s/2}^s\Prob(q\mid\tau=c/s)\,dc=s\int_{1/2}^1\Prob(q\mid\tau)\,d\tau\le\frac s2(9/16)^{s-1}\to0$.
Dominated convergence on $[0,s/2]$ finishes the proof.
\end{proof}

Exact small cases are in \cref{tab:gexact}, and Monte Carlo estimates for \cref{conj:contamination} are in \cref{tab:gmc}.

\begin{proof}[Proof of \cref{prop:pivotmulti}]
We show that the node-ordering (pivot) pipeline of \citet{vesdapunt2014crowdsourcing}, with $K$ clusters of size $s$, has
\[
R_{\mathrm{piv}}=\eps_wK(s-1)^2+\eps_x\binom K2s(2s-1)+O(\eps^2).
\]
In this pipeline the items arrive in uniformly random order.
A new item is compared with each existing block's representative (its first item), following the order in which the blocks were created. It joins the first block that answers ``same''; if all answers are ``different'', it starts a new block.

\emph{Greedy form.}
Each comparison of an item with a representative settles the relation of the item to the whole block.
So the pipeline is a greedy decoder on the partition code with an adaptive order.
Its output is a function of the answers and of the arrival order, which is independent of the answers, and with correct answers it returns the truth.
So \cref{lem:reduction}, with per-coordinate rates, applies.
With only one wrong answer, the run agrees with the noiseless run until that pair is compared. It therefore has the same probability of being queried as in the noiseless run.

\emph{Numbers of comparisons.}
Clusters are created in a uniformly random order.
In the noiseless run, an item of the cluster created $\rho$-th is compared with exactly the $\rho-1$ earlier representatives before it reaches its own representative (or before it creates its cluster).
Since $\rho$ is uniform on $\{1,\dots,K\}$, the expected number of cross comparisons is $n(K-1)/2$.
Every item except the $K$ founders makes exactly one within comparison, so there are $n-K$ within comparisons.

\emph{Damage of a false merge.}
Suppose that the only wrong answer is a false ``same'' between an item $i\in A$ and the representative of the block of another cluster $B$.
Then $i$ joins the block of $B$ and remains there for the rest of the run.
Later items of $A$ are compared with representatives, never with $i$, and the block of $B$ collects all of $B$.
So the damage is $s$ FP pairs ($i$ with the items of $B$) plus $s-1$ FN pairs ($i$ with the other items of $A$), that is, $2s-1$.
This also holds when $i$ is the first item of $A$: the next item of $A$ then creates the block of $A$.

\emph{Damage of a false split.}
Suppose that the only wrong answer is a false ``different'' between an item $i$ and the representative of its own block.
Then $i$ continues with the later representatives, all of which answer ``different'' correctly, and it creates a singleton block.
Later items of its cluster meet the older representative first, receive ``same'', and never reach the block of $i$.
So the damage is $s-1$ FN pairs.

\emph{Summation.}
By \cref{lem:reduction}, $R_{\mathrm{piv}}=\eps_x\cdot\frac{n(K-1)}2\cdot(2s-1)+\eps_w(n-K)(s-1)+O(\eps^2)$ with $n=Ks$.
Since $\frac{n(K-1)}2=\binom K2s$ and $(n-K)(s-1)=K(s-1)^2$, this is the claim.
\end{proof}

\subsection{Pivot-based sorting}\label{app:proofs-orders}

There are $n$ keys, which we identify with their true ranks $1,\dots,n$.
The order code has one coordinate for each of the $\binom n2$ pairs of keys, and its codewords are the $n!$ total orders.
The answer on a pair is its orientation, that is, which of the two keys is smaller.
Each answer is wrong with probability $p$, independently over the pairs.
The loss is the number of pairs that the output order inverts; this is the Hamming distance on the order code.
The raw decoder has risk $p\binom n2$.
Randomized quicksort with noisy comparisons chooses a pivot $q$ uniformly at random from the current list and compares it with every other key of the list, reading the answer on each such pair.
It sends each key to the left or to the right sublist according to the answer, sorts the two sublists recursively, and outputs the left sublist, the pivot and the right sublist, in this order.
We write $H_n=\sum_{k=1}^n1/k$ for the $n$-th harmonic number.

\citet{alonso2004quicksort} established $\E I(n,p)\sim n^2p$ for quicksort with independent comparison errors in the regime $n\to\infty$, $p\to0$ and $np\to\infty$ (their Theorem~2.2 and Remark~2.8).
Below, we derive the exact fixed-$n$ first-order constant in \cref{prop:quicksort} by counting single-error damage and using the classical comparison probabilities.

\begin{proposition}[Quicksort]\label{prop:quicksort}
With comparisons wrong independently with probability $p$, the expected number of inversions of randomized quicksort is $p\,(n^2+3n-2(n+1)H_n)+O(p^2)$, where $H_n$ is the $n$-th harmonic number; the amplification tends to 2.
\end{proposition}

\begin{proof}[Proof of \cref{prop:quicksort}]
We first identify randomized quicksort as a greedy decoder with an adaptive, pivot-first order. We then show that a single wrong comparison of key $i$ with pivot $q$ costs $|i-q|$ inversions. Weighting these costs by their comparison probabilities gives
\[
R_{\mathrm{QS}}(p)=p\,\big(n^2+3n-2(n+1)H_n\big)+O(p^2),\qquad
\frac{n^2+3n-2(n+1)H_n}{\binom n2}\longrightarrow2\quad(n\to\infty),
\]
so that the amplification over the raw decoder tends to~2.

\emph{Quicksort as a greedy decoder.}
The accepted answers compare each chosen pivot with the keys in its current sublist.
They arrange the keys into a sequence of current sublists separated by pivots.
Every total order that keeps this left-to-right arrangement, and orders each current sublist arbitrarily, is consistent with all answers read so far.
Hence the relation between two keys of the same current sublist is not implied, and quicksort queries such a pair if one of the two keys becomes the pivot while both are still in that sublist.
After a pivot $q$ has been compared with every key of its sublist, every pair $(l,r)$ with $l$ and $r$ on opposite sides of $q$ is implied by transitivity through $q$, and it is skipped.
More generally, a pair that quicksort never compares is split by some pivot that was compared with both of its keys, so its relation is implied by transitivity.
Pairs inside a sublist are processed later, in the recursion.
So quicksort accepts the non-implied pairs and skips the implied ones, in an adaptive order: first the pivot against the keys of its sublist, then the pairs across the two new sublists, then recursively inside each sublist.
This is the greedy decoder with an adaptive pivot-first order, and its output is the unique total order consistent with all accepted answers.
Pivots are selected by i.i.d.\ uniform priorities on the keys (the pivot of a sublist is its key of smallest priority), which are independent of the answers; the output is a function of the answers and the priorities, and with correct answers quicksort returns the true order.
So \cref{lem:reduction} applies.

\emph{Damage of a single wrong comparison.}
Suppose that the only wrong answer is on the comparison of key $i$ with pivot $q$, say with $i<q$, so that $i$ is sent to the right sublist.
All other comparisons are correct.
So the right sublist, which contains the keys larger than $q$ together with $i$, is sorted correctly, and $i$ becomes its minimum, placed immediately after $q$.
The left sublist is also sorted correctly.
The inverted pairs are $(i,q)$ and $(i,k)$ for $i<k<q$, that is, $q-i$ pairs.
The case $i>q$ is symmetric, with $i-q$ inverted pairs.
An answer on a pair that quicksort does not compare is never read, and then the damage is~0.

\emph{Probability of a comparison.}
In the noiseless run, the keys of ranks $i<j$ are compared if and only if one of them is the first pivot chosen among the ranks $i,\dots,j$; before that, all these keys stay in one sublist.
By symmetry each of these $j-i+1$ keys is equally likely to be chosen first, so the probability is $2/(j-i+1)$.
With a single wrong answer on the pair $\{i,j\}$, the run coincides with the noiseless run until this answer is read, so the pair is compared with the same probability.

\emph{First-order constant.}
By \cref{lem:reduction}, the first-order constant is
\[
\sum_{i<j}\frac{2}{j-i+1}\,(j-i)
= n(n-1)-\sum_{i<j}\frac{2}{j-i+1}
= n(n-1)-\big[2(n+1)H_n-4n\big]
= n^2+3n-2(n+1)H_n .
\]
Here $\sum_{i<j}\frac2{j-i+1}=\sum_{k=2}^n\frac{2(n+1-k)}k=2(n+1)(H_n-1)-2(n-1)=2(n+1)H_n-4n$, where $k=j-i+1$ and there are $n+1-k$ pairs with $j-i+1=k$.
Dividing by $\binom n2=n(n-1)/2$ and using $H_n=O(\log n)$ shows that the amplification tends to~2.
\end{proof}

\section{Numerical validation}\label{app:numerics}

This appendix gives the simulation methods and results for \cref{tab:checks,tab:highnoise,tab:highnoise-scale,fig:main}. Query orders are uniformly random unless stated otherwise. Reported uncertainties are Monte Carlo standard errors (MCSE): the sample standard deviation divided by the square root of the number of replicates. The code supplement provides the scripts, outputs and commands needed to reproduce each result; random runs use fixed seeds. Simulations run on one CPU core.

\begin{table}[ht]
\caption{Analytical references and simulated damage. Rows~1--4 use exact values (rounded); rows~5--6 use first-order coefficients. Uncertainty is $\pm1$ Monte Carlo standard error (MCSE); $|z|$ is the standardized deviation.}
\label{tab:checks}
\centering
\small
\setlength{\tabcolsep}{4pt}
\begin{tabular}{@{}lll@{}}
\toprule
Quantity & Reference & Simulated \\
\midrule
Conservation, $\E D(a)$ & $\ind\{a\notin B\}$ & $\max|z|=2.66$ \\
Horn hub, $s=2$, $k=4$ & $2.032$ & $2.095\pm0.039$ \\
Horn hub, $s=3$, $k=3$ & $2.314$ & $2.291\pm0.031$ \\
$g(3)$, $K=2,3,5$ & $2.0907$ & $2.099$, $2.106$, $2.083$ \\
$\E\mathrm{FP}/\eps$, $K=4$, $s=3$ & $56.87$ & $58.1\pm1.1$ \\
$\E\mathrm{FN}/\eps$, $K=4$, $s=3$ & $72.03$ & $71.9\pm1.0$ \\
\bottomrule
\end{tabular}
\end{table}

\paragraph{Conservation and Horn hubs.} To check the conservation law, we implement the greedy decoder by enumerating all models that satisfy the theory. After each accepted answer, it retains the consistent models and skips an atom if they all agree on its value. We generate 6 implication theories (16 implications on $n=12$ atoms) and 6 2-CNF theories (14 two-literal clauses, with a unit clause in every second theory). The 11 satisfiable theories contain 132 atoms, including 18 backbone atoms. We make one potential answer wrong at a time and run 600 random orders per atom. Every estimate of $\E_\pi D_\pi(a)$ is within $2.66$ standard errors of the predicted $\ind\{a\notin B\}$. Here $|z|=|\text{mean}-\text{reference}|/\mathrm{MCSE}$, with $z=0$ when MCSE is zero.

For the Horn hub $H_{k,s}$, we run 4000 orders per setting and compare the mean loss at the shared premise with \cref{thm:hub}. For $s=2$ and $k=1,2,3,4$, the estimates are $1.348\pm0.015$, $1.624\pm0.024$, $1.830\pm0.032$ and $2.095\pm0.039$, against exact values $1.333$, $1.600$, $1.829$ and $2.032$. For $s=3$ and $k=1,2,3$, they are $1.522\pm0.014$, $1.897\pm0.023$ and $2.291\pm0.031$, against $1.500$, $1.929$ and $2.314$. We also compare the exact formula with its large-$k$ approximation $\Gamma(1+1/s)\,k^{1-1/s}$: at $k=1024$, the two values are $28.377$ and $28.359$ for $s=2$, and $90.790$ and $90.721$ for $s=3$.

\paragraph{General Horn theories.} We explore \cref{conj:bodysize} beyond the hub using 12 random Horn theories with $n=14$ atoms, 14 to 24 rules and at most $s=2$ or $s=3$ premises per rule. For each atom, we run 300 orders. The cone size $c_a$ counts the newly true atoms, including $a$, when a false atom $a$ is incorrectly set to true and the rules are applied. In these theories, mean damage from such a false positive never exceeds $c_a^{1-1/s}$, with equality only at cone size~1. False-negative damages reach $1.48$.

We also test two-level trees with rules $a\wedge b_i\to c_i$ and $c_i\wedge d_{ij}\to e_{ij}$. For $k=2,4,6$, the mean damage of a wrong answer at $a$ is $0.77$, $0.61$ and $0.57$ times $\sqrt{c_a}$. In three larger random theories with at most two premises per rule ($n=60$, $60$ and $80$), the largest ratio to $\sqrt{c_a}$ is $1.36$. The tested families are consistent with the square-root growth suggested by the $s=2$ hub.

\paragraph{One cluster.} We compute first-order amplification from the query tree as $\E\,W(\MST_m)/\binom m2$, its mean path length over all record pairs. For random orders, we generate random weight matrices and find their minimum spanning trees. At $m=50$, $100$, $200$, $400$, $800$, $1600$ and $3200$, using 200, 100, 50, 20, 40, 40 and 40 matrices respectively, the estimates are $6.965\pm0.053$, $9.729\pm0.104$, $13.393\pm0.248$, $17.526\pm0.423$, $23.662\pm0.428$, $32.033\pm0.645$ and $41.887\pm0.751$. The ratios to $m^{1/3}$ rise from $1.89$ to $2.84$, so these sizes do not determine the asymptotic exponent. The reference line $c\,m^{1/3}$ in \cref{fig:main}(B) uses $c=2.377$, the geometric mean of the seven ratios.

For similarity-first, we draw $m$ uniform points in $[0,1]^d$ and compute the Euclidean minimum spanning tree, which is the noiseless query tree for this order. In one dimension, amplification is $(m+1)/3$. In two dimensions, it is $16.52\pm0.30$, $41.50\pm1.06$, $98.52\pm1.39$ and $235.26\pm4.42$ at $m=100$, $400$, $1600$ and $6400$ (40, 20, 40 and 40 point sets). These values are $3.56$, $5.63$, $8.42$ and $12.67$ times $m^{1/3}$. The largest three random-order sizes and largest two 2D sizes each use 40 independent replicates, with seed sequences $(20261006,0,m)$ and $(20261006,2,m)$ respectively.

\paragraph{Two clusters.} We compute the exact values of $g(s_A,s_B)$ in \cref{tab:gexact} by a recursive calculation with rational arithmetic, caching states already visited. A state records the current blocks and the constraints between them. Under a random query order, the next pair that can change this state is uniform among the unconstrained pairs between distinct blocks.

For the Monte Carlo estimates of $g(s)$ in \cref{tab:gmc}, a run stops when the block containing the wrongly merged pair can no longer change; \cref{lem:anatomy} ensures that the final damage is then determined. At $s=2,3,4$, the estimates are $1.6001\pm0.0079$, $2.1065\pm0.0145$ and $2.4867\pm0.0212$, close to the exact values. The column $s\,\Prob(q)$ measures the chance that the potential wrong pair is actually queried, rescaled by cluster size. We evaluate \cref{lem:anatomy}(1) using a dynamic program for the component sizes of $G(s,\tau)$ and Gauss--Legendre quadrature. The values approach the limiting constant $C_q\approx1.4286$ in \cref{prop:cq}, computed by numerical integration. At $s=10$, 40000 orders give $g(10)=4.43$. The random-to-pivot false-merge loss ratio $g(s)s/(2s-1)$ equals $29/27\approx1.074$, $1.254$ and $1.430$ for $s=2,3,4$ (exact) and about $2.3$, $4.5$ and $11.5$ at $s=10$, $32$ and $128$ (simulation).

The estimates in \cref{tab:gmc} give $g(s)/s^{2/3}=0.90\pm0.02$, $0.87\pm0.03$, $0.89\pm0.06$ and $0.90\pm0.15$ at $s=16$, $32$, $64$ and $128$. They support $s^{2/3}$ growth but do not determine the asymptotic constant.

\begin{table}[h]
\centering
\caption{Exact mean damage of one false merge between two clusters of sizes $s_A$ and $s_B$, split into wrongly merged pairs ($g_{\mathrm{FP}}$) and wrongly separated pairs ($g_{\mathrm{FN}}$), with $g=g_{\mathrm{FP}}+g_{\mathrm{FN}}$. The mean is over random query orders, with zero damage when the false pair is skipped. Decimals are rounded to six places.}
\label{tab:gexact}
\begin{tabular}{cccc}
\toprule
$(s_A,s_B)$ & $g_{\mathrm{FP}}$ & $g_{\mathrm{FN}}$ & $g$\\
\midrule
$(1,1)$ & $1$ & $0$ & $1$\\
$(1,2)$ & $1$ & $1/3$ & $4/3$\\
$(2,2)$ & $46/45$ & $53/90$ & $29/18=1.611111$\\
$(1,3)$ & $1$ & $28/45$ & $73/45=1.622222$\\
$(2,3)$ & $1.032959$ & $0.830710$ & $10567/5670=1.863668$\\
$(3,3)$ & $1.053091$ & $1.037571$ & $313913/150150=2.090663$\\
$(2,4)$ & $1.037609$ & $1.058191$ & $2.095800$\\
$(3,4)$ & $1.064236$ & $1.239428$ & $2.303664$\\
$(4,4)$ & $1.081504$ & $1.420525$ & $2.502029$\\
\bottomrule
\end{tabular}
\end{table}

\begin{table}[h]
\centering
\caption{Two clusters of size $s$: Monte Carlo estimates of $g(s)$ (mean $\pm1$ Monte Carlo standard error over the given number of random orders), the numerically evaluated rescaled probability $s\,\Prob(q)$ that the false pair is queried, and $g/s^{2/3}$.}
\label{tab:gmc}
\begin{tabular}{rrccc}
\toprule
$s$ & Orders & $g(s)$ & $s\,\Prob(q)$ & $g/s^{2/3}$\\
\midrule
8 & 40000 & $3.879\pm0.050$ & 1.176 & 0.970\\
16 & 40000 & $5.719\pm0.112$ & 1.244 & 0.901\\
32 & 30000 & $8.790\pm0.295$ & 1.298 & 0.872\\
64 & 16000 & $14.20\pm0.93$ & 1.340 & 0.888\\
128 & 5000 & $22.9\pm3.9$ & 1.369 & 0.901\\
\bottomrule
\end{tabular}
\end{table}

\paragraph{Several clusters.} To test false-merge locality, we plant one false merge between two of $K=2,3,5$ clusters of size~3 and run 20000 orders for each $K$. The mean damages are $2.099\pm0.021$, $2.106\pm0.021$ and $2.083\pm0.021$, compared with the common prediction $g(3)=2.0907$ from \cref{thm:multicluster}. We then run the full pipeline on $K=4$ clusters of size~3, with independent answer errors at rate $\eps=10^{-3}$. Over 200000 runs, $\E\,\mathrm{FP}/\eps=58.1\pm1.1$ and $\E\,\mathrm{FN}/\eps=71.9\pm1.0$. Dividing by the error rate makes these estimates comparable to the first-order coefficients: $56.87$ and $72.03$, computed from the exact $g_{\rm FP}(3)$, $g_{\rm FN}(3)$ and $\E\,W(\MST_3)=4$ in \cref{thm:multicluster}. Because these runs allow multiple errors, the normalized expectations can differ from the first-order coefficients by $O(\eps)$.

\paragraph{Finite noise.} To test the orders when several answers may be wrong in one run, \cref{tab:highnoise,tab:highnoise-scale} use one cluster with $m=100$ or $m=200$ records. Every answer is independently flipped to ``different'' with probability $\delta\in\{0.01,0.02,0.05,0.1,0.2\}$. Pivot first asks all pairs $\{0,j\}$ involving a fixed record and then processes the remaining pairs in random order. The random order uniformly permutes all pairs. Each cell uses 200 runs, with seeds 20260929 and 20261006 for $m=100$ and $m=200$, respectively. All 2000 runs at $m=200$ satisfy the loss bounds of \cref{lem:sandwich}.

\begin{table}[ht]
\centering
\caption{Scale check for finite noise: fraction of wrongly separated pairs in one cluster of $m=200$ records. Mean $\pm1$ Monte Carlo standard error over 200 runs per cell; Random/Pivot is the ratio of mean losses.}
\label{tab:highnoise-scale}
\begin{tabular}{llll}
\toprule
$\delta$ & Pivot & Random & Random/Pivot\\
\midrule
0.01 & $0.020\pm0.001$ & $0.106\pm0.009$ & 5.4\\
0.02 & $0.039\pm0.001$ & $0.200\pm0.012$ & 5.1\\
0.05 & $0.097\pm0.002$ & $0.361\pm0.010$ & 3.7\\
0.10 & $0.179\pm0.003$ & $0.502\pm0.007$ & 2.8\\
0.20 & $0.337\pm0.003$ & $0.646\pm0.005$ & 1.9\\
\bottomrule
\end{tabular}
\end{table}

\newpage
\section{Real-LLM experiment}\label{app:realpilot}

\paragraph{Data and model answers.} We use HPI's 1,879-record Cora file and duplicate-pair list \citep{hpicora}. Duplicate-pair connected components define 182 entities, one per paper. Records use ``field: value'' pairs in the fixed order authors, title, journal, booktitle, institution, publisher, volume, pages, date and year. Qwen2.5-1.5B (base) \citep{qwen25model} receives six hand-written citation-pair examples (three Yes, three No; none from Cora), then the test pair with the smaller record identifier first. It answers ``same'' when $\mathrm{logit}(\texttt{ Yes})>\mathrm{logit}(\texttt{ No})$. These next-token scores are used without sampling or fine-tuning. Similarity-first uses cosine similarity of normalized CLS embeddings from bge-small-en-v1.5 \citep{bgesmallmodel}.

\paragraph{Experimental settings.} The one-entity experiment uses the 23 entities with at least 20 records: 1,266 records, sizes 21 to 238, and all 62,423 within-entity pairs. For the multi-entity experiment, we group these entities into six blocks of three or four with nearby embedding centroids, mimicking similarity blocking. This includes comparisons between records from similar but distinct entities. Each entity contributes at most 30 records; blocks contain 87 to 120 records and 36,151 pairs in total. Pivot uses the multi-entity pipeline of \cref{prop:pivotmulti}.

All query orders reuse the same saved model answers. Primary pivot and random comparisons average over 200 orders per instance. The natural/shuffled error-placement control uses 30 matched order draws. Shuffled errors are uniformly redistributed within each entity. Similarity-first randomly breaks cosine ties: once per entity in the primary comparison and 200 times per multi-entity block. Amplification divides final wrong-pair counts by wrong answers across all candidate pairs, including skipped pairs. The 95\% bootstrap intervals resample entities (or blocks) with model answers fixed.

\paragraph{One-entity results.} At threshold 0, final wrong-pair fractions are 8.1\% for pivot, 15.3\% for random order and 5.7\% for similarity-first. Pivot reduces this fraction by $7.2$ percentage points relative to random order (pivot-minus-random interval $[-18.7,-1.4]$). Similarity-first's lower point estimate differs from pivot by $-2.5$ points, with interval $[-9.6,6.0]$ including zero. \Cref{tab:sweep} varies the Yes--No score-difference threshold to change the answer-error rate. At every tested rate, the pivot-minus-random interval lies below zero. In the error-placement control of \cref{tab:realllm}, the 30 matched draws give natural-placement amplifications of 1.37 (pivot), 2.19 (random) and 0.85 (similarity-first).

\begin{table}[h]
\centering
\caption{Amplification in the one-entity setting as the decision threshold changes. Wrong: wrong answers as a percentage of the 62,423 pairs; 6.6\% is the threshold 0 of \cref{tab:realllm}.}
\label{tab:sweep}
\begin{tabular}{cccc}
\toprule
Wrong & Pivot & Random & Similarity-first\\
\midrule
1.1\% & 1.48 & 4.85 & 1.18\\
2.1\% & 1.40 & 3.75 & 1.05\\
5.1\% & 1.26 & 2.64 & 0.71\\
6.6\% & 1.23 & 2.31 & 0.85\\
10.0\% & 1.22 & 1.99 & 1.21\\
20.1\% & 1.20 & 1.61 & 1.00\\
\bottomrule
\end{tabular}
\end{table}

\paragraph{Multi-entity results.} Across the six blocks, the model gives 1,444 wrong ``same'' answers and 846 wrong ``different'' answers. The final wrong-pair counts are 2,603 for pivot, 3,347 for random order and 2,759 for similarity-first. The pivot-minus-random loss-fraction difference is $-2.1$ percentage points, with interval $[-3.4,-0.7]$. Similarity-first produces the most false merges: 1,727 wrongly merged pairs, compared with 1,414 for pivot. Its priority for similar pairs also brings similar cross-entity pairs forward.

Exploratory runs with SmolLM2-1.7B-Instruct gave almost exclusively ``different'' answers and the same loss for all orders. DBLP--ACM entities contain only two records, leaving no within-entity propagation to compare. These runs are not used in the reported comparisons.

\paragraph{Assets.} Qwen2.5 is released under the Apache 2.0 license and bge-small-en-v1.5 under the MIT license; the HPI page provides the Cora data with attribution. The supplementary bundle contains the cached model scores and embeddings, the analysis code and the per-instance results.

\end{document}